\documentclass[12pt]{article}
\usepackage{amsmath,xcolor}
\usepackage{graphicx}
\usepackage{enumerate}
\usepackage{natbib}
\usepackage{amsmath,amsfonts,amssymb}
\usepackage{algorithm}
\usepackage{url} 
\usepackage{bm}
\usepackage{appendix}
\usepackage{amsmath,amsfonts,amssymb}
\newcommand{\blind}{1}
\newtheorem{theorem}{Theorem}
\newtheorem{lemma}{Lemma}
\newtheorem{proof}{Proof}[section]

\newtheorem{definition}{Definition}

\def\RR{{\mathbb{R}}}
\def\U{{\mathcal{U}}}
\def\S{{\mathcal{S}}}

\def\pzx{\widehat{p}(\bz)}
\def\pzxr{{{\widehat{p}^*(\bz)}}}
\def\pzxp{{{\widehat{p}^+(\bz)}}}

\newcommand{\X}{\mathcal{X}}
\newcommand{\x}{\bm{x}}

\newcommand{\ba}{\bm{a}}
\newcommand{\0}{\mathbf{0}}
\newcommand{\bu}{\bm{u}}
\newcommand{\bs}{\bm{s}}
\newcommand{\bz}{\bm{z}}
\newcommand{\bX}{\mathbf{X}}
\newcommand{\bH}{\mathbf{H}}

\begin{document}

\def\spacingset#1{\renewcommand{\baselinestretch}%
{#1}\small\normalsize} \spacingset{1}

\newcommand{\red}[1]{\textcolor{red}{#1}}
\newcommand{\blue}[1]{\textcolor{blue}{#1}}

\addtolength{\textheight}{.5in}%

\if1\blind
{
  \title{\bf An optimal transport approach for selecting a representative subsample with application in efficient kernel density estimation}
  \date{}
  \author{Jingyi Zhang \\
    Center for Statistical Science, Tsinghua University\\
    \\
    Cheng Meng \footnote{Joint first author} \\
    Center for Applied Statistics, \\ Institute of Statistics and Big Data, Renmin University of China\\
    \\
    Jun Yu \\
    School of Mathematics and Statistics, Beijing Institute of Technology
    \\
    \\
    Mengrui Zhang, Wenxuan Zhong, and Ping Ma\footnote{Corresponding author}\\
    Department of Statistics, University of Georgia.
    }
  \maketitle
} \fi





\spacingset{1.5}
\begin{abstract}
Subsampling methods aim to select a subsample as a surrogate for the observed sample.
Such methods have been used pervasively in large-scale data analytics, active learning, and privacy-preserving analysis in recent decades.
Instead of model-based methods, in this paper, we study model-free subsampling methods, which aim to identify a subsample that is not confined by model assumptions.
Existing model-free subsampling methods are usually built upon clustering techniques or kernel tricks. 
Most of these methods suffer from either a large computational burden or a theoretical weakness.
In particular, the theoretical weakness is that the empirical distribution of the selected subsample may not necessarily converge to the population distribution.
Such computational and theoretical limitations hinder the broad applicability of model-free subsampling methods in practice.
We propose a novel model-free subsampling method by utilizing optimal transport techniques. 
Moreover, we develop an efficient subsampling algorithm that is adaptive to the unknown probability density function.
Theoretically, we show the selected subsample can be used for efficient density estimation by deriving the convergence rate for the proposed subsample kernel density estimator.
We also provide the optimal bandwidth for the proposed estimator.
Numerical studies on synthetic and real-world datasets demonstrate the performance of the proposed method is superior.
\end{abstract}

\noindent%
\vfill

\newpage
\section{Introduction}\label{sec:intro}
A subsampling problem can be described as follows: given a $d$-dimensional sample $\{\x_i\}_{i=1}^n$ generated from an unknown probability distribution, the goal is to take a subsample $\{\x^*_i\}_{i=1}^r$, $r\ll n$, as a surrogate for the original sample.
In recent decades, the subsampling problem has drawn great attention in machine learning, statistics, and computer science. 
For example, subsampling methods are used pervasively in optimal design/active learning problems, where in a large sample of unlabeled data, the goal is to select an informative subsample to label \citep{settles2012active}.
Consider privacy-preserving analysis as another example.
In some applications, subsampling methods have the potential to enhance data security \citep{nissim2007smooth, li2012sampling}. 
Specifically, a carefully selected subset of data can reveal little confidential information \citep{shu2015privacy}.
Last but not least, subsampling methods are also widely applied in algorithm design to alleviate the computational burden in large-scale data analysis \citep{tsai2015big,zhou2017machine}.

Many existing subsampling methods are model-based methods, which assume predictors and responses, if any, follow a postulated model. 
These methods aim to select an informative subsample that benefits model-fitting and prediction.
Various models have been considered in subsampling problems, including linear regression
\citep{drineas2006fast,drineas2011faster,ma2014statistical,ma2015statistical,ma2015leveraging,wang2017computationally,meng2017effective,zhang2018statistical,ma2020asymptotic,li2020modern}, generalized linear regression \citep{wang2018optimal,ai2021boptimal,yu2020optimal}, $l_p$ regression \citep{dasgupta2009sampling}, quantile regression \citep{ai2021optimal}, streaming time series model \citep{xie2019online}, Gaussian mixture model \citep{feldman2011scalable}, nonparametric regression \citep{meng2020more,meng2021smoothing}, among others \citep{bardenet2017markov,quiroz2018speeding,yu2022subdata}.
While model-based subsampling methods have already yielded impressive achievements, the key to the success of these methods highly depends on the correct model specification.
Nevertheless, in practice, model specification is a trial and error process, and a postulated model for the data could be misspecified.
For example, in supervised learning, we start with a high dimensional model with numerous features; and by using model selection, we may end up with a low dimensional model with parsimonious features. 
In another instance, we may start with a linear regression model for a continuous response; and by discretizing the response, we may end up with a classification model.
Model-based subsampling methods, however, may result in subsamples hampering such dynamic processes of model specification \citep{tsao2012subsampling}. 
Consequently, in scenarios when the model may be misspecified or in the stage of exploratory analysis, more preferred methods are model-free subsampling methods, which can identify a subsample that is not confined by model assumptions.

Recently, there have been emerging model-free subsampling methods, which aim to select a representative subsample that can capture the overall patterns of the observed sample.
These methods can be divided into two classes: clustering-based approaches and kernel-based approaches.
Clustering-based approaches, which are usually used in unsupervised learning methods, 
include $k$-medoids method \citep{kaufman1987clustering,park2009simple}, $k$-center method \citep{feder1988optimal}, and Wasserstein barycenter method \citep{agueh2011barycenters,cuturi2014fast}.
The $k$-medoids method is closely related to the $k$-means algorithm, and the $k$-center method is used extensively in fast multipole methods \citep{greengard1991fast,white1994continuous,yang2003improved,lee2009fast}.
The Wasserstein barycenter method aims to find the barycenter of a set of empirical probability measures under the optimal transport metric, and such a barycenter itself can be regarded as a representative subsample. 
Despite wide applications of these subsampling methods, the empirical distributions of the selected subsamples, yielded by these clustering-based approaches, may not resemble the probability distribution of the original sample.
That is, as the subsample size increases, the probability distributions of the subsample identified by these methods may not necessarily converge to the true probability distribution.
To address such a limitation, 
researchers developed 
kernel-based approaches, which aim to select a subsample that can effectively approximate the population distribution.
These approaches include the kernel herding method \citep{chen2014data}, the coreset for kernel density estimation \citep{phillips2013varepsilon,zheng2013quality,zheng2017visualization}, and the support point method \citep{mak2018support}.
Despite the theoretical benefits, one limitation of these kernel-based approaches is that they may result in a large computational burden in large-scale data analysis.

To overcome the computational and theoretical limitations of the aforementioned methods, we propose a novel model-free subsampling method that is computationally efficient and enjoys nice theoretical properties. 
The proposed method combines the techniques of optimal transport and space-filling designs.
In particular, we first transform the observed sample to be uniformly distributed on a hypercube using optimal transport techniques \citep{villani2008optimal,peyre2019computational}, then select a set of data points that can effectively represent the uniform distribution using space-filling designs \citep{owen2003quasi,fang2005design}.
The desired subsample is the one corresponding to the selected data points.
The idea is analogous to an inverse procedure of the inverse transform sampling technique, which transforms a uniformly distributed sample to a sample that follows an arbitrary probability density function.
Theoretically, we show the proposed subsample kernel density estimator converges to the true probability density function under mild conditions.
Moreover, we show the proposed estimator converges faster than the estimator based on a randomly selected subsample, suggesting the proposed method can be utilized for efficient density estimation.
We also provide the optimal bandwidth for the proposed estimator.
Numerically, utilizing projection-based optimal transport methods \citep{pitie2005n,rabin2011wasserstein}, the computational cost for the proposed method is at the order of $O(n\log(n)d^2)$ for a $d$-dimensional sample of size $n$.
The proposed method thus is scalable to datasets with large $n$ and moderate $d$.
Numerical studies on synthetic and real-world datasets demonstrate the superior performance of the proposed method in comparison with mainstream competitors.
The proposed method is implemented in an R package, named \texttt{SPARTAN}.


\section{Preliminaries}

\subsection{Star discrepancy and space-filling designs}
The proposed method is developed upon the notion of star discrepancy, which is a classical metric that measures the discrepancy between a set of discrete data points and the uniform distribution on the unit hypercube $[0,1]^d$, denoted by $U[0,1]^d$ \citep{niederreiter1992random,fang1993number,fang2005design}.
Let $1\{\cdot\}$ be the indicator function and $\ba=(a_1,\ldots,a_d)\in[0,1]^d$ be a vector.
Let $[\0, \ba)=\prod_{j=1}^d[0,a_j)$ be a hyper-rectangle and $\U_r=\{\bu_i\}_{i=1}^r$ be a set of $r$ data points in $[0,1]^d$.
We introduce the definition of the star discrepancy in the following. 

\begin{definition}
Given $\U_r$ and a hyper-rectangle $[\0,\ba)$, $\ba\in[0,1]^d$, the corresponding local discrepancy is defined as,
$D(\U_r,\ba)=|\frac{1}{r}\sum^r_{i=1}1\{\bu_i\in[\0,\ba)\}-\prod^d_{j=1}a_j|.$
The star discrepancy is defined as
$$D^*(\U_r)=\underset{\ba\in [0,1]^d}{\sup}D(\U_r,\ba).$$
\end{definition}

Definition~1 suggests a set of data points $\U_r$, which can effectively represent $U[0,1]^d$, has a small value of $D^*(\U_r)$, and vice versa.
There exist methods that generate design points via directly minimizing the star discrepancy, and these methods are called uniform design methods \citep{fang2005design}.
Despite wide applications, most of these methods are computationally expensive and are not scalable to a design with a large number of points.
To alleviate such a computational burden, methods yielding a set of design points with a relatively small star discrepancy could be used as alternatives for uniform design methods. 
These alternatives include space-filling design methods \citep{wu2011experiments,fang2005design} and low-discrepancy sequences \citep{owen2003quasi,lemieux2009book,dick2013high,leobacher2014introduction}.
The former aims to generate a set of design points that spread out over the domain as uniformly as possible.
The latter sequentially generates the design points, which achieve an asymptotically fast decay rate respecting the star discrepancy.
Consequently, these methods provide powerful tools to generate a set of representative design points in terms of $U[0,1]^d$.

We now discuss the theoretical property of space-filling designs and low-discrepancy sequences in terms of the star discrepancy \citep{owen2003quasi}. 
For a Sobol sequence $\S_r=\{\bs_i\}_{i=1}^r$, a representative of low-discrepancy sequences, $D^*(\S_r)$ converges to zero at the rate of $O(\log(r)^d/r)$.
In other words, the convergence rate of $D^*(\S_r)$ is of the order $O(r^{-(1-\delta)})$ for an arbitrary small $\delta>0$ and fixed $d$, as $r$ goes to infinity. 
For comparison, when a set of data points $\X_r=\{\x_i\}_{i=1}^r$ is randomly generated from $U[0,1]^d$, the convergence rate of $D^*(\X_r)$ is of the order $O((\log\log(r)/r)^{1/2})$, which is much slower than $O(r^{-(1-\delta)})$ \citep{chung1949estimate}.
By adopting a method which is no worse than the Sobol sequence, in this paper, we always assume the star discrepancy $D^*(\S_r)$ converges to zero with the rate $O(r^{-(1-\delta)})$.
There also exist some space-filling designs that can achieve a potentially faster convergence rate in terms of star discrepancy \citep{fang2005design}.

Utilizing space-filling design techniques, we propose a simple algorithm to select a representative subsample from a sample that is generated from $U[0,1]^d$. 
Let $\{\bu_i\}_{i=1}^n$ be such a sample.
The proposed algorithm, summarized in Algorithm~1, combines space-filling design techniques and the one-nearest-neighbor approximation.

\begin{algorithm}
\caption{Select a representative subsample from a sample generated from $U[0,1]^d$.}
\begin{tabbing}
   \qquad \enspace \textit{Step 1.} Generate a set of space-filling design points $\{\bs_i\}_{i=1}^r\in [0,1]^d$\\
   \qquad \enspace \textit{Step 2.} For $i=1$ to $r$\\
   \qquad \qquad \qquad \enspace Select the nearest neighbor for $s_i$ from $\{\bu_i\}_{i=1}^n$ using the Euclidean distance\\
   \qquad \qquad \qquad \enspace Let $u^*_i$ be the selected data point \\
   \qquad \enspace \textit{Step 3.}  The final subsample is given by $\U^*_r = \{\bu^*_i\}_{i=1}^r$ 
\end{tabbing}\label{alg:ALG1}
\end{algorithm}

Lemma 1 below, which is first stated in \cite{meng2020more}, characterizes the approximation error of the subsample selected by Algorithm 1.
This lemma suggests the selected subsample can effectively approximate the design points in the sense that their corresponding star discrepancies are almost at the same order under certain conditions.

\begin{lemma}\label{lem} 
Let $\S_r=\{\bs_i\}_{i=1}^r\in[0,1]^d$ be a set of design points which satisfy $D^*(\S_r)=O(r^{-(1-\delta)})$ for any arbitrary small $\delta>0$, as $r\rightarrow\infty$.
Suppose $d$ is fixed, when $r=O(n^{1/d})$, as $n \rightarrow \infty$, we have  $D^*(\U^*_r)=O_p(r^{-(1-\delta)}).$
\end{lemma}

Algorithm~1 can be extended to the case that the cumulative distribution function $F$ of the samples is non-uniform when $d=1$.
The idea is analogous to the classical inverse transform sampling method \citep{devroye1986sample,mosegaard1995monte}.
Let $\{x_i\}_{i=1}^n\in\RR$ be the observed sample, we first calculate $\{F(x_i)\}_{i=1}^n$, from which, we then select a subsample $\{F(x_i^*)\}_{i=1}^r$ using Algorithm \ref{alg:ALG1}.
Notice that the transformed sample is uniformly distributed on $[0,1]$; thus, the selected subsample is relatively representative of $U[0,1]$.
Finally, the desired subsample is given by $\{x_i^*\}_{i=1}^r$. 
Although this simple strategy works well in practice, a limitation of such a strategy is that it is inapplicable when $d\geq2$ \footnote{One exception is that when all the covariates of the sample are independent with each other, in which case one can directly calculate the multivariate cumulative distribution function as the product of all the one-dimensional marginal cumulative distribution function. Nevertheless, independent covariates are rarely the case in practice.}. 
To overcome the limitation, we introduce the optimal transport map, which serves as a surrogate for $F$ in multivariate cases.
This idea is similar to the one in \cite{chernozhukov2017monge}, where the authors used the optimal transport map to extend the concepts of quantiles and ranks from one-dimensional samples to multivariate samples.
Analogously, in this paper, we use the optimal transport map to extend the technique of inverse transform sampling from one-dimensional cases to high-dimensional cases.

\subsection{Optimal transport maps}
Optimal transport maps have been extensively used as a standard technique to transform one probability distribution to another.
Recently, such maps have received a significant attention in machine learning and computer science \citep{ferradans2014regularized,rabin2014adaptive,su2015optimal,courty2017optimal,meng2020sufficient,peyre2019computational}, due to its close relationship with generative models, including generative adversarial nets \citep{goodfellow2014generative}, the ``decoder'' network in variational autoencoders \citep{kingma2013auto}, among others.

Instead of introducing the general definition of the optimal transport map, we now present a specific map of our interest, and we refer to \cite{villani2008optimal, peyre2019computational,zhang2021review} for more details.
Let $u$ be the uniform probability distribution on $[0,1]^d$.
Let $p_X$ and $\Omega\subseteq\mathbb{R}^d$ be the probability distribution and the domain of the random variable $X$, respectively.
Let $\#$ be the push-forward operator, such that for all measurable $B\subset \Omega$, we have $\phi_\#(p_X)(B)=p_X(\phi^{-1}(B))$.
Among all the maps $\phi:\Omega\rightarrow[0,1]^d$ such that $\phi_\#(p_X) = u$ and $\phi^{-1}_\#(u) = p_X$, the optimal transport map $\phi^*$ of our interest is the one that minimizes the $L_2$ cost, $\int_\Omega \|X- \phi(X) \|^2 \mbox{d}p_X$, where $\|\cdot\|$ denotes the Euclidean norm.
We focus on $L_2$ cost in this paper for simplicity and it is possible to consider other costs as long as the optimal transport map exists.
For the $L_2$ cost, as a special case, when $\Omega=\mathbb{R}$ and $d=1$, it is known that $\phi^*$ is equivalent to the cumulative distribution function $F$ \citep{villani2008optimal}.
This fact motivates us to use the $\phi^*$ as a surrogate for $F$ in high-dimensional cases.

To obtain the desired optimal transport map that maps the observed sample to be uniformly distributed on $[0,1]^d$, we propose to first generate a synthetic sample from $U[0,1]^d$, then calculate the optimal transport map from the observed sample to the synthetic sample.
One can utilize the auction algorithm or the refined auction algorithm to calculate such a map \citep{bertsekas1992auction,transport2020}.
Despite the effectiveness, the auction algorithm has an average computational cost of the order $O(n^2)$, and thus it may incur an enormous computational cost when $n$ is large.
To alleviate the computational burden, in practise, we propose to approximate the optimal transport map $\phi^*$ using projection-based methods \citep{pitie2007automated,bonneel2015sliced,rabin2011wasserstein,meng2019large,zhang2022projection}.
These methods tackle the problem of estimating a $d$-dimensional optimal transport map iteratively by breaking down the problem into a series of subproblems. 
Each of the subproblems involves finding a one-dimensional optimal transport map between the projected samples, and such a subproblem can be easily solved through sorting algorithms.
\section{Main algorithm}
We develop a novel subsampling method named SPARTAN, which integrates space-filling design techniques and optimal transport methods.
The proposed method works as follows. 
First, we transform the observed sample, denoted by $\{\x_i\}_{i=1}^n$, to be uniformly distributed on $[0,1]^d$.
We achieve this goal by utilizing the empirical optimal transport map.
Here, the empirical optimal transport map is also called the optimal matching between two discrete distributions, such that each of them have $n$ atoms and each atom has weight $1/n$.
We use such an empirical optimal transport map as a surrogate of the optimal transport map between the underlying population density function of the observed sample and the uniform distribution.
We then select a set of data points of size $r$ from the transformed sample using Algorithm~\ref{alg:ALG1}.
The subsample corresponding to the selected data points is the final output.
We summarize the algorithm below.

\begin{algorithm}
\caption{Space-filling after optimal transport (SPARTAN)}
\begin{tabbing}
   \qquad \enspace \textit{Step 1.} Generate a synthetic random sample $\{\bu_i\}_{i=1}^n$ from $U[0,1]^d$ \\
   \qquad \enspace \textit{Step 2.} Calculate the empirical optimal transport map, denoted by $\widehat{\phi}$, that maps \\ 
   \qquad \qquad \enspace the observed sample $\{\x_i\}_{i=1}^n$ to the synthetic sample $\{\bu_i\}_{i=1}^n$\\
   \qquad \enspace \textit{Step 3.}  Calculate the transformed sample $\{\widehat{\phi}(\x_i)\}_{i=1}^n$\\
   \qquad \enspace \textit{Step 4.}  Select a set of data points $\{\widehat{\phi}(\x_i^*)\}_{i=1}^r$ from $\{\widehat{\phi}(\x_i)\}_{i=1}^n$ using Algorithm \ref{alg:ALG1}\\
   \qquad \enspace \textit{Step 5.}  The final subsample is given by $\{\x_i^*\}_{i=1}^r.$
\end{tabbing}\label{alg:ALG2}
\end{algorithm}

Figure~\ref{fig1} illustrates Algorithm~\ref{alg:ALG2} using a toy example. 
A two-dimensional synthetic sample of size 1000,  marked as grey dots, is shown in Fig.~\ref{fig1}(a). 
We first transform the sample to be uniformly distributed on $[0,1]^2$ using the projection pursuit Monge map method \citep{meng2019large}, shown in Fig.~\ref{fig1}(b).
We then generate 32 design points using a space-filling design method \citep{owen2003quasi,fang2005design}. 
The design points are marked as triangles in Fig.~\ref{fig1}(c).
Next, for each design point, we search for its nearest neighbor, labeled as black dots in Fig.~\ref{fig1}(c).
Finally, the subsample corresponding to the selected data points, marked as black dots in Fig.~\ref{fig1}(d), gives the desired subsample. 

\begin{figure}[ht]\centering
        \includegraphics[width=\textwidth]{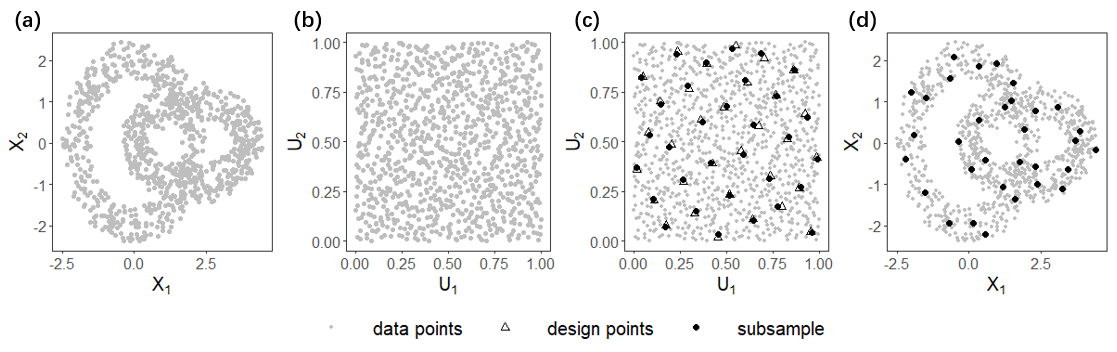}
        \caption{Illustration for Algorithm~\ref{alg:ALG2}. The two-dimensional sample, marked as gray dots in panel (a), is first transformed to be uniformly distributed on $[0,1]^2$, shown in panel (b). We then generate a set of space-filling design points, marked as  triangles, and search for the nearest neighbor for each of them, marked by black dots in panel (c). Panel (d) shows the subsample corresponding to the selected data points.} \label{fig1}
\end{figure}

The computational cost for Algorithm~\ref{alg:ALG2} mainly incurs in Step~2 and Step~4.
In particular, we use a projection-based method to approximate the desired optimal transport map in Step~2, requiring a computational cost of the order $O(n\log(n)d^2)$ \citep{pitie2007automated,bonneel2015sliced,meng2019large}.
Step~4 includes two sub-steps: generating the design points and searching the corresponding nearest neighbors.
The design points can be generated beforehand; thus, the computation time for generating these points is not considered here.
For searching the nearest neighbors, we opt to use the $k$-d tree method, whose computation cost is at the order of $O(n\log(n))$ \citep{bentley1975multidimensional,wald2006building}. 
In sum, the overall computational complexity for Algorithm~\ref{alg:ALG1} is at the order of $O(n\log(n)d^2)$.

Figure~\ref{fig2} visualizes the subsamples (black dot) selected by the proposed method (lower row) compared with the subsamples selected by the random subsampling method (upper row).
The two-dimensional samples (grey dots) are generated from three different distributions: the standard Gaussian distribution (left column), a mixture Gaussian distribution (middle column), and a mixture beta distribution (right column).
From plots in the left column, one can observe that the randomly selected subsample is far from symmetric.
From plots in the middle and the right columns, one can see that some peaks in the probability distribution are largely overlooked by the random subsampling method.
We observe that the subsamples identified by the proposed method have a more robust and appealing visual representation of the corresponding probability distribution in all the cases.

\begin{figure}[ht]\centering
        \includegraphics[width=\textwidth]{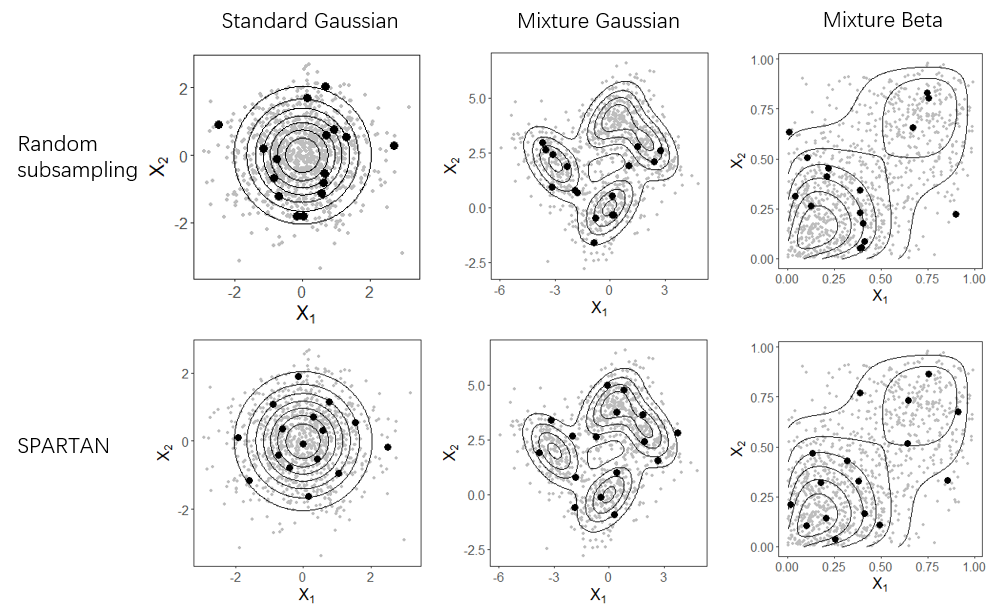}
        \caption{Subsamples (black dots) selected by the proposed method (lower) versus randomly selected subsamples (upper). Contours (black) are superimposed. One can observe the proposed method selects subsamples that have more appealing visual representation of the corresponding population.} \label{fig2}
\end{figure}

\section{Theoretical results}
In this section, we study the theoretical properties of the subsamples obtained in Algorithm~\ref{alg:ALG2}. In particular, we develop an asymptotic theory concerning the rates of convergence of the estimated density to the true density as the sample size goes to infinity. 
The rates are calculated in terms of the point-wise mean squared error (MSE) that defined as
$\mbox{MSE}(\widehat{p}(\bz))=  E\{\widehat{p}(\bz)-p(\bz) \}^2$, where $\bz\in\RR^d$, $\widehat{p}$ is the density estimator and ${p}$ is the true density. 
The density is estimated using the widely-used kernel density estimation method.
Throughout this paper, we consider the Gaussian kernel.
The extension of the main theorem to other kernel functions is straightforward, as long as such a kernel function satisfies some regularity conditions, which are relegated to the Supplementary Material.
A more in-depth discussion on different choices of kernel functions can be found in \cite{scott2015multivariate}.
To avoid trivial cases, we consider the case that $d\ge 2$ in this section.
Without lose of generality, we assume the points $\{\x_i\}$ are distinct, and the points $\{\bu_i\}$ are distinct.
In such cases, the optimal transport map in Step~2 of Algorithm~2 is a one-to-one map from $\{\x_i\}_{i=1}^n$ to $\{\bu_i\}_{i=1}^n$.
Let $p$ be the probability density function to be estimated. 
Two widely-used regularity conditions for $p$ are required in kernel density estimation,
\begin{itemize}
    \item Condition (a). $\partial^2 p(z)/\partial z_j^2$ is absolutely continuous, for $j=1,\ldots,d$,;
    \item Condition (b). $\partial^3 p(z)/\partial z_j^3$ is square-integrable, for $j=1,\ldots,d$.
\end{itemize}

Let $\bX\in \RR^{n\times d}$ be the sample matrix, where the $(i,j)$-th element is $x_{ij}$, and $\bX^*\in \RR^{r\times d}$ be the subsample matrix, where the $(i,j)$-th element is $x_{ij}^*$.
Let $h>0$ be the bandwidth and $K:\RR\rightarrow\RR$ be a kernel function.
For any $\bz\in \RR^d$, the full-sample product kernel density estimator can be written as
\begin{eqnarray}\label{eqn:prod}
\pzx = \sum_{i=1}^n \left[\prod_{j=1}^d K\left\{(z_j-x_{ij})/h\right\}/h\right]/n.
\end{eqnarray}

Equation~(\ref{eqn:prod}) can be generalized to a more general multivariate kernel density estimator.
In particular, for a $d\times d$ nonsingular bandwidth matrix $\bH$ and a multivariate kernel function $\mathcal{K}:\mathbb{R}^d\rightarrow \mathbb{R}$, a general multivariate kernel estimator can be written as
\begin{eqnarray}\label{eqn:general}
\widehat{p}_{general}(z) = \frac{1}{n|\bH|}\sum_{i=1}^n \left[ \mathcal{K}\left\{\bH^{-1}(\bz-\x_i)\right\}\right].
\end{eqnarray}
It is apparent that Equation~(\ref{eqn:general}) is equivalent to Equation~(\ref{eqn:prod}) when $\bH=h\cdot \mathbf{I}_d$, where $\mathbf{I}_d$ is the identity matrix.
Let $\mathcal{K}$ be the Gaussian kernel in Equation~(\ref{eqn:general}), it is equivalent to choose $\mathcal{K}=\mathcal{N}(\0,\bm{\Sigma})$ with $\bH=\mathbf{I}_d$, or to choose $\mathcal{K}=\mathcal{N}(\0,\mathbf{I}_d)$ with $\bH=\bm{\Sigma}^{1/2}$ in Equation~(\ref{eqn:prod}). 
Consequently, with a properly chosen kernel function, one can reformulate a general multivariate kernel estimator to a product kernel density estimator.
We thus only focus on the product kernel density estimator in this section without loss of generality.

Analogous to Equation~(\ref{eqn:prod}), the density estimator $\pzxr$ that computed from the subsample can be written as
\begin{eqnarray*}
\pzxr = \sum_{i=1}^r \left[\prod_{j=1}^d K\left\{(z_j-x_{ij}^*)/h\right\}/h\right]/r.
\end{eqnarray*}
We derive the convergence rate for the mean squared error for the proposed subsample estimator.
The results are summarized in Theorem 1 below, and the proof is relegated to Appendix.

\begin{theorem}
Suppose 
$p$ satisfies Conditions (a) and (b).
Moreover, suppose $p$ has a compact convex domain $\Omega \subset \mathbb{R}^d$, and  there exists a constant $c\ge 1$ for which $c^{-1}\le p(\x)\le c$ for any $\x\in \Omega$.
When $d\ge 2, r=O(n^{1/d})$, as $n\rightarrow\infty$ and $h\rightarrow 0$, for any arbitrary small $\delta>0$, we have
$$\mbox{MSE}(\pzxr) = O{\left(\frac{1}{r^{2(1-\delta)} h^{d+2}}\right)} + O(h^4).$$
In particular, if $h=O(r^{-2(1-\delta)/(d+6)})$, we have
\begin{eqnarray}\label{maineqn}
\mbox{MSE}(\pzxr)= O(r^{-8(1-\delta)/(d+6)}).
\end{eqnarray}
\end{theorem}

Theorem~1 shows the proposed subsample estimator converges to the true probability density function.
Moreover, Theorem~1 indicates the proposed subsampling method can be used for efficient density estimation.
Specifically, let $\bX^+\in\RR^{r\times d}$ be a randomly selected subsample matrix, and $\pzxp$ be the corresponding subsample estimator.
According to Theorem 6.4 of \cite{scott2015multivariate}, as $r=o(n)$ and $n\rightarrow\infty$, when $h=O(r^{-1/(4+d)})$, $\mbox{MSE}(\pzxp)$ achieves the optimal convergence rate $O(r^{-4/(d+4)})$ for any $z\in\Omega$.
Such a convergence rate is much slower than the convergence rate in Equation~(\ref{maineqn}).
Consequently, Theorem~1 indicates one can approximate the probability density function $p$ more efficiently using the proposed subsample kernel density estimator, compared with the counterpart based on a randomly selected subsample.

Consider the bandwidth $h$, or generally, the bandwidth matrix $\bH\in\mathbb{R}^{d\times d}$.
In practice, one can determine the value of $\bH$ through the plug-in approach or the cross-validation approach \citep{duong2003plug,chacon2010multivariate,scott2015multivariate}.
One limitation of these approaches, however, is that they may result in a computational burden for the sample with moderate or large $n$.
To combat the computational burden, we opt to determine the value of $\bH$ using the general Scott's rule \citep{scott2015multivariate}, which suggests to use $\bH=r^{-1/(d+4)}\times\widehat{\bm{\Sigma}}^{1/2}$ for a subsample kernel density estimator that based on a subsample of size $r$. 
Here, $\widehat{\bm{\Sigma}}$ is the empirical variance-covariance matrix for the observed sample.
Analogously, as suggested by Theorem~1, we also consider using $\bH=r^{-2/(d+6)}\times\widehat{\bm{\Sigma}}^{1/2}$ for the proposed estimator. 
Consider the essential condition in Theorem~1, which requires the domain of $p$ to be compact convex.
Empirically, we find the proposed estimator still works reasonably well when such a condition does not hold, as shown in the following section.

\section{Simulation Results}
To evaluate the proposed subsampling method, we compare it with three mainstream competitors in terms of the estimation accuracy of the kernel density estimator.
The competitors include the uniform subsampling method, also called the random subsampling method, the $k$-medoids method, and the support point method \citep{mak2018support}.
We use the projection-pursuit Monge map method \citep{meng2019large} for approximating the optimal transport map in Algorithm~2.
All the methods are implemented in \texttt{R}, and all the parameters are set as default.

For each subsampling method, we first calculate the subsample kernel density estimator $\widehat{p}(\x)$, then evaluate the accuracy of which using the Hellinger distance \citep{li2016density}, defined as
$1-\sum_{i=1}^n\sqrt{\widehat{p}(x_i)/p(x_i)}/n,$
where $\{\x_i\}_{i=1}^{n}$ is an independent testing dataset generate from the same probability density function as the training sample.
Empirically, we find other metrics, like the mean squared error considered in Theorem~1, also yield similar performance.
For the kernel density estimator, we use the Gaussian kernel and the general Scott's rule \citep{scott2015multivariate} to determine the bandwidth matrix.
In particular, for all the subsample estimator, the bandwidth matrix $\bH=r^{-1/(d+4)}\times\bm{\widehat\Sigma}^{1/2},$ where $\bm{\widehat\Sigma}$ is the empirical variance-covariance matrix.
For the proposed method, we also consider the cases that $\bH=r^{-2/(d+6)}\times\bm{\widehat\Sigma}^{1/2},$ according to Theorem~1.
The standard errors are calculated through a hundred replicates.
In each replicate, we generate a synthetic training sample with $n=10^4$ from $d=\{2,5,10,20\}$ and each of the following three probability density functions,
\begin{itemize}
    \item \texttt{D1}: A Gaussian distribution $\mathcal{N}(\0,\bm{\Sigma})$, where $\bm{\Sigma}_{ij} = 0.5^{|i-j|}$, $i,j=1,...,d$;
    \item \texttt{D2}: A mixture Gaussian distribution \\ 
    $\mathcal{N}(\mathbf{1}, \bm{\Sigma})/4+\mathcal{N}(\mathbf{-1}, \bm{\Sigma})/4+\mathcal{N}(\0, \bm{\Sigma})/2$, where $\bm{\Sigma}=0.8^{|i-j|}$, $i,j=1,...,d$.
    \item \texttt{D3}: A mixture $t$-distribution, whose degree-of-freedom equals 8,10, and 12, \\
    $t(\0, \bm{\Sigma}, 8)/3+t(\0, \bm{\Sigma}, 10)/3+t(\0, \bm{\Sigma}, 12)/3$, where $\bm{\Sigma}=0.8^{|i-j|}$, $i,j=1,...,d$.
\end{itemize}

\begin{figure}\centering
        \includegraphics[width=\textwidth]{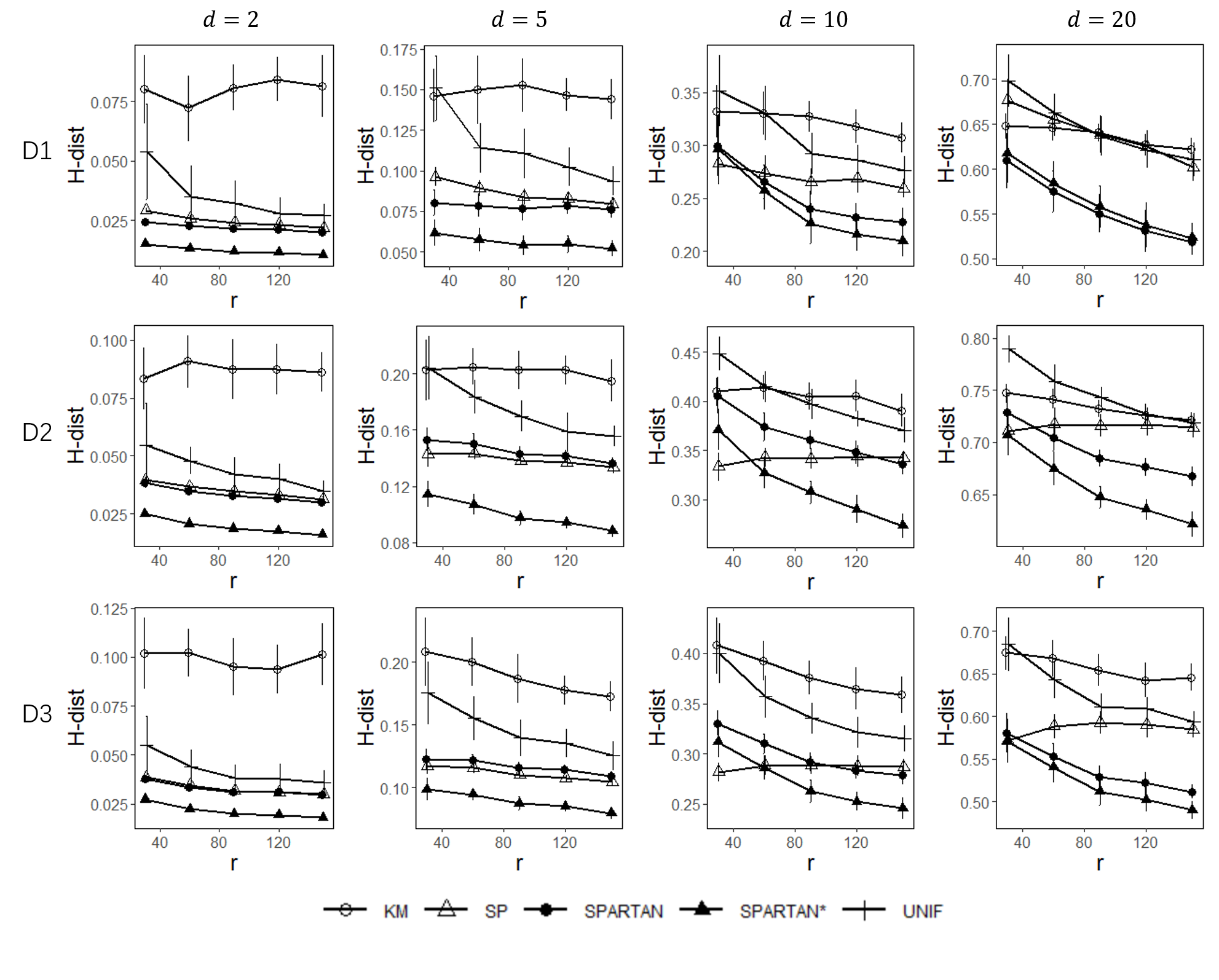} 
        \caption{Simulation under different $d$ (from left to right) and different probability density functions (from upper to lower). 
        The Hellinger distance (H-dist) are plotted versus different $r$.
        Vertical bars represent the standard errors.
        } \label{fig3}
\end{figure}

Figure~\ref{fig3} shows the Hellinger distance versus different $r$ under various settings. 
Each row represents a particular data distribution \texttt{D1}--\texttt{D3}, and each column represents a particular $d$.
We use crosses to denote the uniform subsampling method (UNIF), hollow circles to denote the K-medoids method (KM), hollow triangles to denote the support point method (SP), solid circles to denote the proposed method (SPARTAN), and solid triangles to denote the proposed method with $\bH=r^{-2/(d+6)}\times\bm{\widehat\Sigma}^{1/2}$ (SPARTAN$^*$).

Three significant observations can be made from Fig.~\ref{fig3}.
We first observe that the K-medoids method performs worse than the uniform subsampling method in almost all cases. 
Moreover, the support point method outperforms the uniform subsampling method in all cases.
We also observe the Hellinger distance yielded by these two methods do not converge to zero in some cases.
Such an observation can be attributed to the fact that the probability distribution of the subsample identified by these two methods may not necessarily converge to the true probability distribution.

Second, we observe the Hellinger distance yielded by the proposed method decreases as $r$ increases.
Moreover, the proposed method outperforms the uniform subsampling method in all cases.
These observations are consistent with Theorem~1, which indicates the proposed subsample estimator converges to the true probability density function and is more efficient than the estimator corresponding to the uniform subsampling method.

Third, we observe the proposed estimator with $\bH=r^{-1/(d+4)}\times\bm{\widehat\Sigma}^{1/2}$ outperforms the other three competitors in most of the cases.
As the same bandwidth matrices are applied in all these estimators, such a comparison is fair.
Consequently, the aforementioned observation suggests the subsample identified by the proposed subsampling method is more representative of the observed sample than the subsamples selected by the other three methods.
We also observe the proposed estimator with $\bH=r^{-2/(d+6)}\times\bm{\widehat\Sigma}^{1/2}$ consistently outperforms the one with $\bH=r^{-1/(d+4)}\times\bm{\widehat\Sigma}^{1/2}$.
This observation is consistent with Theorem~1, which suggests $h=O(r^{-2(1-\delta)/(d+6)})$ yields the smallest upper bound of the asymptotic integrated mean squared error for the proposed estimator.

\section{Real data example}
\subsection{Density estimation}
Throughout this section, we consider the banknote authentication dataset, which is extracted from images that were taken from 1372 genuine and forged banknotes. 
Wavelet transform was used to extract four features from images \footnote{The dataset can be downloaded from https://archive.ics.uci.edu/ml/datasets/banknote+authentication}.
To evaluate the performance of the proposed subsampling method, we compare it with other competitors in terms of the accuracy of the kernel density estimation and the prediction accuracy in active learning. 
A brief introduction to active learning will be given later.

We first visualize the banknote authentication dataset and the subsample selected by the proposed method.
In Fig.\ref{fig5}, the lower diagonal panels show the scatter plots for each pair of the predictors. 
We select a subsample of size fifty, and the scatter plots for such a subsample are shown in the upper panels of Fig.\ref{fig5}.
The heat maps are obtained using kernel density estimation.
We observe the selected subsample has an appealing visual representation of the original sample.

\begin{figure}\centering
        \includegraphics[width=0.9\textwidth]{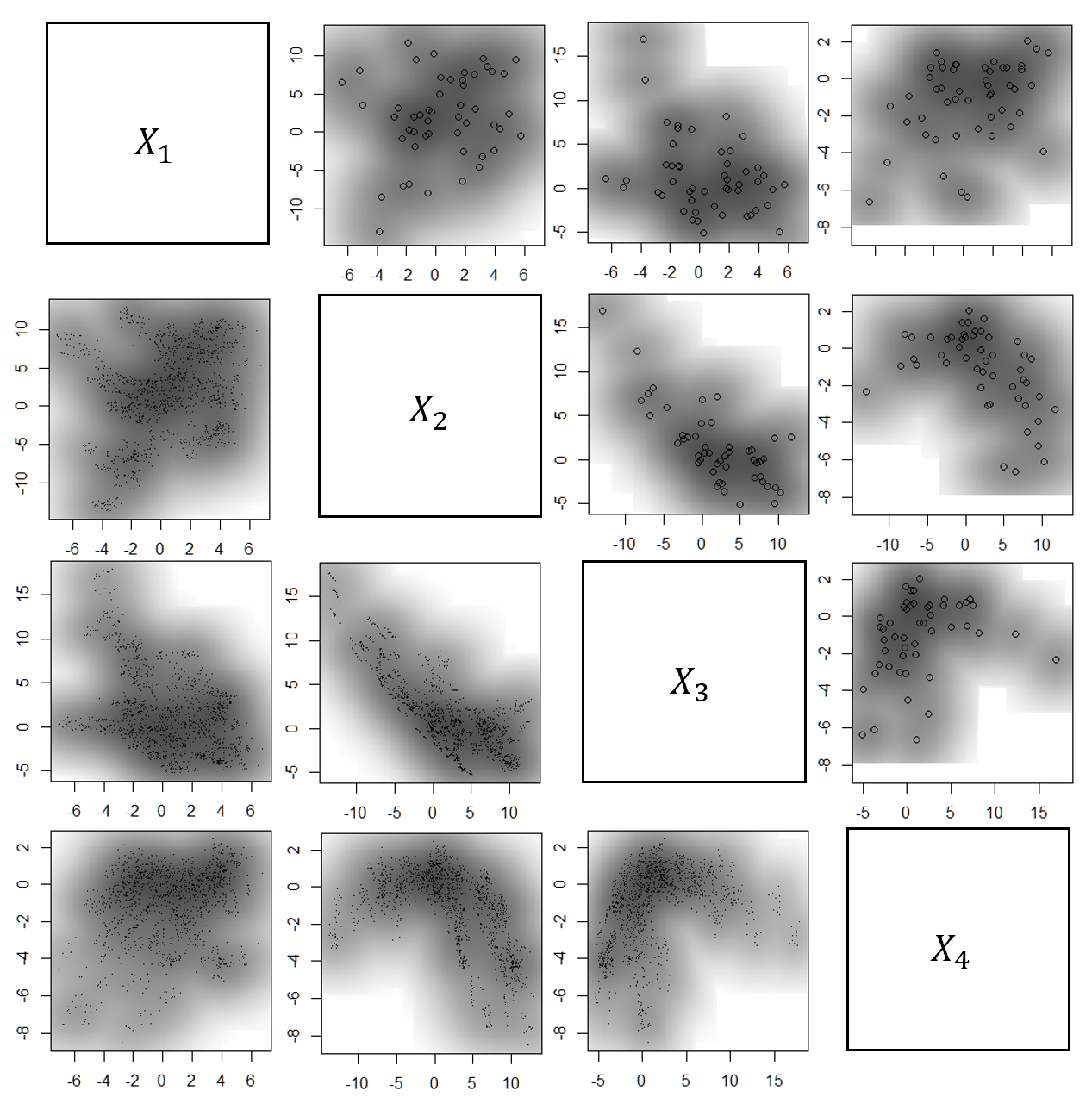}
        \caption{Visualization of the banknote authentication dataset. The lower diagonal panels show the scatter plots for each pair of predictors. The upper diagonal panels show the scatter plots for the selected subsample using the proposed algorithm. The heat maps are obtained using kernel density estimation.
        } \label{fig5}
\end{figure}

For density estimation, we consider three competitors, as mentioned in the previous section.
Same as the settings stated in the previous section, we used the Gaussian kernel for kernel density estimators and the general Scott’s rule to determine the value of the bandwidth matrix.
All the parameters are set as the same as the ones we used in the previous section.
We replicated the experiment twenty times.
In each replication, the dataset is randomly divided into the training set and the testing set of equal sizes.
We first calculate the full sample kernel density estimator using the testing set, denoted by $\widehat{p}_{full}$.
For each subsample kernel density estimator, we then evaluate its estimation accuracy through the empirical Hellinger distance, defined as
$1-\sum_{i=1}^n\sqrt{\widehat{p}(x_i)/\widehat{p}_{full}(x_i)}/n,$
where $\{\x_i\}_{i=1}^n$ represents the testing set.
This empirical Hellinger distance is not a formal distance and thus may have negative values, as we will see later.
Nevertheless, the empirical Hellinger distance can be used as a surrogate for the true Hellinger distance since a small value of the empirical Hellinger distance is associated with a small value of the true Hellinger distance, intuitively.

\begin{figure}\centering
        \includegraphics[width=\textwidth]{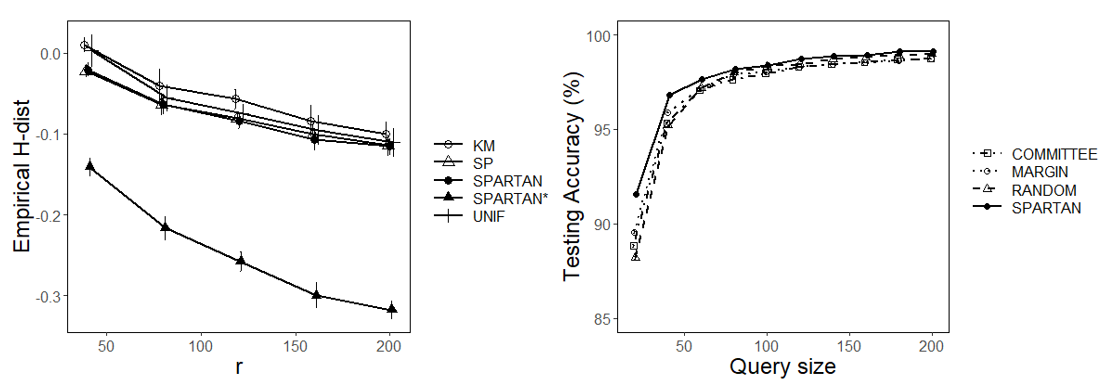} \\
        \caption{\footnotesize{Left: For the density estimation of the banknote authentication dataset, the empirical Hellinger distance (H-dist) is plotted versus different subsample sizes $r$.
        Right: For the active learning of the banknote authentication dataset, the testing accuracy is plotted versus different query sizes.
        Vertical bars represent the standard errors. In the right panel, the standard errors are tiny, and thus the error bars are almost invisible.}
        } \label{fig4}
\end{figure}

The left panel of Fig.~\ref{fig4} shows the empirical Hellinger distance versus different subsample sizes $r$. 
The standard error bars are obtained from one hundred replicates.
We observe that the uniform subsampling method consistently outperforms the K-medoids method.
We then observe that the proposed method and the support point method perform similarly, and both have better performance than the uniform subsampling method.
Finally, we observe the proposed estimator with $\bH=r^{-2/(d+6)}\times\bm{\widehat\Sigma}^{1/2}$, as guided by Theorem~1, gives the best result.
All these observations are consistent with the findings in the previous section.

\subsection{Active learning}
We now consider the task of active learning, which aims to make an accurate prediction, with the number of labeled training data points as small as possible \citep{krogh1995neural,cohn1996active}.
These approaches are essential for numerous sophisticated supervised learning tasks, where the labeled instances are challenging, time-consuming, or expensive to obtain. 
Take speech recognition as an example; accurate labeling of speech utterances is extremely time-consuming and requires trained linguists. 
It is reported that annotation at the level of the phoneme can take 400 times longer than the actual audio \citep{settles2012active}.
In general, active learning approaches select the data points (also termed as the query points) iteratively and interactively.
In each iteration, one query the oracle to obtain the label at a new query point, based on certain criteria. 
It is known that a representative subsample is potentially associated with an accurate prediction in active learning \citep{settles2012active}.


The proposed subsampling method can be cast as an active learning approach. 
In particular, we generate the Sobol sequence \citep{owen2003quasi} in Algorithm~1 and select the query points sequentially in Algorithm~2.
To evaluate the performance of the proposed method, we compare it with the following baseline methods: (1) random sampling (RANDOM), (2) query by committee (COMMITTEE), which select query points that maximize the disagreement among different models \citep{settles2012active}, and (3) margin-based method (MARGIN) which choose query points that lie on the margin of the decision line \citep{schohn2000less}.

We replicate the experiment a hundred times on the banknote dataset. 
In each replication, the dataset is randomly divided into the training set and the testing set of equal sizes. 
We evaluate the classification model by its mean classification accuracy on the testing set.
The classification accuracy is defined as $(TP+FN)/n$, where $n$ denotes the size of the testing set, and $TP$ and $FN$ denote true positive and false negative, respectively.
We use the support vector machine, implemented by the R package \texttt{e1071} \citep{meyer2015e1071}), for classification in the active learning.
The RBF kernel with default parameters is applied. 
The size of query points ranges from 10 to 200.
For the committee method and the margin-based method, which require several initial labeled data points as input, ten data points are randomly selected and labeled.

The right panel of Fig.~\ref{fig4} shows the mean classification accuracy of different active learning methods versus different numbers of query points.
The vertical bars represent the standard errors.
These bars, however, are almost invisible due to extremely small values of standard errors.
We observe the proposed method consistently outperforms all the competitors.
We attribute such an observation to the fact that the proposed method selects a representative subsample in a sequential way, resulting in a more accurate prediction in active learning.

\section{Discussion}

In this paper, we proposed a novel model-free subsampling method, utilizing the space-filling design and optimal transport techniques.
The proposed algorithm is efficient and can be adaptive to the unknown probability density function.
Theoretically, we show the proposed subsample kernel density estimator converges to the true probability density function under mild conditions.
The order for the optimal smoothing parameter for the proposed kernel density estimator is also derived.
The superior performance of the proposed method over mainstream competitors was justified by various numerical experiments.

In this paper, we mainly focus on using the unit cube as the target distribution due to mathematically simplicity.
In practise, it is possible to consider standard Gaussian distribution instead.
Specifically, we could generate the random sample from the standard Gaussian distribution in Algorithm~\ref{alg:ALG2}, and use the Gaussian Sobol sequence instead of the space-filling design points in Algorithm~\ref{alg:ALG1}.
The other steps remain the same.
Empirical results show such a scheme may lead to slightly better performance.
The proposed method has the potential to be applied to many large-sample applications, including but not limited to nonparametric regression, kernel methods, and low-rank approximation of matrices.
This work may speed up these researches with theoretical guarantees.

\section*{Acknowledgment}
The authors thank the associate editor and two anonymous reviewers for provided helpful comments on earlier drafts of the manuscript.
The authors would like to acknowledge the support from National Key R\&D Program of China (No. 2021YFA1001300), National Natural Science Foundation of China Grant No.12101606, No.12001042, the U.S. National Science Foundation under grants DMS-1903226, DMS-1925066, the U.S. National Institute of Health under grant R01GM122080, and Beijing Institute of Technology research fund program for young scholars.

\section*{Conflict of Interest}
The authors report there are no competing interests to declare.

\clearpage
\bibliographystyle{agsm}
\bibliography{ref}

@inproceedings{rabin2011wasserstein,
  title={Wasserstein barycenter and its application to texture mixing},
  author={Rabin, Julien and Peyr{\'e}, Gabriel and Delon, Julie and Bernot, Marc},
  booktitle={International Conference on Scale Space and Variational Methods in Computer Vision},
  pages={435--446},
  year={2011},
  organization={Springer}
}

@article{yu2022subdata,
  title={Subdata selection algorithm for linear model discrimination},
  author={Yu, Jun and Wang, HaiYing},
  journal={Statistical Papers},
  pages={1--24},
  year={2022},
  publisher={Springer}
}

@incollection{meng2017effective,
  title={Effective statistical methods for big data analytics},
  author={Meng, Cheng and Wang, Ye and Zhang, Xinlian and Mandal, Abhyuday and Zhong, Wenxuan and Ma, Ping},
  booktitle={Handbook of research on applied cybernetics and systems science},
  pages={280--299},
  year={2017},
  publisher={IGI Global}
}

@article{zhang2022projection,
  title={Projection-based techniques for high-dimensional optimal transport problems},
  author={Zhang, Jingyi and Ma, Ping and Zhong, Wenxuan and Meng, Cheng},
  journal={Wiley Interdisciplinary Reviews: Computational Statistics},
  pages={e1587},
  year={Just accepted},
  publisher={Wiley Online Library}
}

@article{quiroz2018speeding,
  title={Speeding up MCMC by efficient data subsampling},
  author={Quiroz, Matias and Kohn, Robert and Villani, Mattias and Tran, Minh-Ngoc},
  journal={Journal of the American Statistical Association},
  year={2018},
  publisher={Taylor \& Francis}
}

@article{bardenet2017markov,
  title={On Markov chain {M}onte {C}arlo methods for tall data},
  author={Bardenet, R{\'e}mi and Doucet, Arnaud and Holmes, Chris},
  journal={The Journal of Machine Learning Research},
  volume={18},
  number={1},
  pages={1515--1557},
  year={2017},
  publisher={JMLR. org}
}

@article{chernozhukov2017monge,
  title={Monge--{K}antorovich depth, quantiles, ranks and signs},
  author={Chernozhukov, Victor and Galichon, Alfred and Hallin, Marc and Henry, Marc},
  journal={The Annals of Statistics},
  volume={45},
  number={1},
  pages={223--256},
  year={2017},
  publisher={Institute of Mathematical Statistics}
}

@InProceedings{nissim2007smooth,
  author       = {Nissim, Kobbi and Raskhodnikova, Sofya and Smith, Adam},
  booktitle    = {Proceedings of the thirty-ninth annual ACM symposium on Theory of computing},
  title        = {Smooth sensitivity and sampling in private data analysis},
  organization = {ACM},
  pages        = {75--84},
  year         = {2007},
}

@article{ai2021optimal,
  title={Optimal subsampling algorithms for big data regressions},
  author={Ai, Mingyao and Yu, Jun and Zhang, Huiming and Wang, HaiYing},
  journal={Statistica Sinica},
  volume={31},
  pages={1--24},
  year={2021}
}

@Article{mak2018support,
  author    = {Mak, Simon and Joseph, V Roshan},
  title     = {Support points},
  number    = {6A},
  pages     = {2562--2592},
  volume    = {46},
  journal   = {The Annals of Statistics},
  publisher = {Institute of Mathematical Statistics},
  year      = {2018},
}

@Article{tsai2015big,
  author    = {Tsai, Chun-Wei and Lai, Chin-Feng and Chao, Han-Chieh and Vasilakos, Athanasios V},
  title     = {Big data analytics: a survey},
  number    = {1},
  pages     = {21},
  volume    = {2},
  journal   = {Journal of Big data},
  publisher = {SpringerOpen},
  year      = {2015},
}

@Article{su2015optimal,
  author    = {Su, Zhengyu and Wang, Yalin and Shi, Rui and Zeng, Wei and Sun, Jian and Luo, Feng and Gu, Xianfeng},
  title     = {Optimal mass transport for shape matching and comparison},
  number    = {11},
  pages     = {2246--2259},
  volume    = {37},
  journal   = {IEEE transactions on pattern analysis and machine intelligence},
  publisher = {IEEE},
  year      = {2015},
}

@Article{li2020modern,
  author    = {Li, Tao and Meng, Cheng},
  title     = {Modern Subsampling Methods for Large-Scale Least Squares Regression},
  number    = {2},
  pages     = {1--28},
  volume    = {2},
  journal   = {International Journal of Cyber-Physical Systems (IJCPS)},
  publisher = {IGI Global},
  year      = {2020},
}

@Article{wang2018optimal,
  author    = {Wang, HaiYing and Zhu, Rong and Ma, Ping},
  title     = {Optimal subsampling for large sample logistic regression},
  number    = {522},
  pages     = {829--844},
  volume    = {113},
  journal   = {Journal of the American Statistical Association},
  publisher = {Taylor \& Francis},
  year      = {2018},
}

@Book{niederreiter1992random,
  author    = {Niederreiter, Harald},
  title     = {Random number generation and quasi-{M}onte {C}arlo methods},
  publisher = {SIAM},
  year      = {1992},
}

@Article{zhou2017machine,
  author    = {Zhou, Lina and Pan, Shimei and Wang, Jianwu and Vasilakos, Athanasios V},
  title     = {Machine learning on big data: Opportunities and challenges},
  pages     = {350--361},
  volume    = {237},
  journal   = {Neurocomputing},
  publisher = {Elsevier},
  year      = {2017},
}

@InProceedings{wald2006building,
  author       = {Wald, Ingo and Havran, Vlastimil},
  booktitle    = {Interactive Ray Tracing 2006, IEEE Symposium on},
  title        = {On building fast $k$d-trees for ray tracing, and on doing that in ${O}(N logN)$},
  organization = {IEEE},
  pages        = {61--69},
  year         = {2006},
}

@InProceedings{ma2014statistical,
  author    = {Ma, Ping and Mahoney, Michael and Yu, Bin},
  booktitle = {International Conference on Machine Learning},
  title     = {A statistical perspective on algorithmic leveraging},
  pages     = {91--99},
  year      = {2014},
}

@article{zhang2021review,
  title={A review on modern computational optimal transport methods with applications in biomedical research},
  author={Zhang, Jingyi and Zhong, Wenxuan and Ma, Ping},
  journal={Modern Statistical Methods for Health Research},
  pages={279--300},
  year={2021},
  publisher={Springer}
}

@Article{tsao2012subsampling,
  author    = {Tsao, Min and Ling, Xiao},
  title     = {Subsampling method for robust estimation of regression models},
  number    = {03},
  pages     = {281},
  volume    = {2},
  journal   = {Open Journal of Statistics},
  publisher = {Scientific Research Publishing},
  year      = {2012},
}

@InProceedings{zheng2013quality,
  author       = {Zheng, Yan and Jestes, Jeffrey and Phillips, Jeff M and Li, Feifei},
  booktitle    = {Proceedings of the 2013 ACM SIGMOD International Conference on Management of Data},
  title        = {Quality and efficiency for kernel density estimates in large data},
  organization = {ACM},
  pages        = {433--444},
  year         = {2013},
}

@Book{scott2015multivariate,
  author    = {Scott, David W},
  title     = {Multivariate Density Estimation: Theory, Practice, and Visualization},
  publisher = {John Wiley \& Sons},
  year      = {2015},
}

@Article{ma2015leveraging,
  author    = {Ma, Ping and Sun, Xiaoxiao},
  title     = {Leveraging for big data regression},
  number    = {1},
  pages     = {70--76},
  volume    = {7},
  journal   = {Wiley Interdisciplinary Reviews: Computational Statistics},
  publisher = {Wiley Online Library},
  year      = {2015},
}

@Article{chacon2010multivariate,
  author    = {Chac{\'o}n, Jos{\'e} E and Duong, T},
  title     = {Multivariate plug-in bandwidth selection with unconstrained pilot bandwidth matrices},
  number    = {2},
  pages     = {375--398},
  volume    = {19},
  journal   = {Test},
  publisher = {Springer},
  year      = {2010},
}

@Article{wang2017computationally,
  author  = {Wang, Yining and Yu, Adams Wei and Singh, Aarti},
  title   = {On Computationally Tractable Selection of Experiments in Measurement-Constrained Regression Models},
  number  = {143},
  pages   = {1--41},
  volume  = {18},
  journal = {Journal of Machine Learning Research},
  year    = {2017},
}

@Article{chen2014data,
  author    = {Chen, CL Philip and Zhang, Chun-Yang},
  title     = {Data-intensive applications, challenges, techniques and technologies: A survey on Big Data},
  pages     = {314--347},
  volume    = {275},
  journal   = {Information Sciences},
  publisher = {Elsevier},
  year      = {2014},
}

@InProceedings{rabin2014adaptive,
  author       = {Rabin, Julien and Ferradans, Sira and Papadakis, Nicolas},
  booktitle    = {2014 IEEE International Conference on Image Processing (ICIP)},
  title        = {Adaptive color transfer with relaxed optimal transport},
  organization = {IEEE},
  pages        = {4852--4856},
  year         = {2014},
}

@Article{pitie2007automated,
  author    = {Piti{\'e}, Fran{\c{c}}ois and Kokaram, Anil C and Dahyot, Rozenn},
  title     = {Automated colour grading using colour distribution transfer},
  number    = {1-2},
  pages     = {123--137},
  volume    = {107},
  journal   = {Computer Vision and Image Understanding},
  publisher = {Elsevier},
  year      = {2007},
}

@Article{meng2021smoothing,
  author    = {Meng, Cheng and Yu, Jun and Chen, Yongkai and Zhong, Wenxuan and Ma, Ping},
  title     = {Smoothing splines approximation using Hilbert curve basis selection},
  number    = {just-accepted},
  pages     = {1--26},
  journal   = {Journal of Computational and Graphical Statistics},
  publisher = {Taylor \& Francis},
  year      = {2021},
}

@Article{greengard1991fast,
  author    = {Greengard, Leslie and Strain, John},
  title     = {The fast {G}auss transform},
  number    = {1},
  pages     = {79--94},
  volume    = {12},
  journal   = {SIAM Journal on Scientific and Statistical Computing},
  publisher = {SIAM},
  year      = {1991},
}

@Article{trillos2014rate,
  author    = {Trillos, Nicol{\'a}s Garcia and Slep{$\Breve{c}$}ev, Dejan},
  title     = {On the rate of convergence of empirical measures in infinity-transportation distance},
  number    = {6},
  pages     = {1358--1383},
  volume    = {67},
  journal   = {Canadian Journal of Mathematics},
  publisher = {Cambridge University Press},
  year      = {2015},
}

@Book{villani2008optimal,
  author    = {Villani, C{\'e}dric},
  title     = {Optimal transport: old and new},
  publisher = {Springer Science \& Business Media},
  year      = {2008},
}

@Article{bentley1975multidimensional,
  author    = {Bentley, Jon Louis},
  title     = {Multidimensional binary search trees used for associative searching},
  number    = {9},
  pages     = {509--517},
  volume    = {18},
  journal   = {Communications of the ACM},
  publisher = {ACM},
  year      = {1975},
}

@InProceedings{zheng2017visualization,
  author       = {Zheng, Yan and Ou, Yi and Lex, Alexander and Phillips, Jeff M},
  booktitle    = {2017 IEEE Visualization in Data Science (VDS)},
  title        = {Visualization of Big Spatial Data using Coresets for Kernel Density Estimates},
  organization = {IEEE},
  pages        = {23--30},
  year         = {2017},
}

@Book{fang2005design,
  author    = {Fang, Kai-Tai and Li, Runze and Sudjianto, Agus},
  title     = {Design and modeling for computer experiments},
  publisher = {CRC Press},
  year      = {2005},
}

@InProceedings{phillips2013varepsilon,
  author       = {Phillips, Jeff M},
  booktitle    = {Proceedings of the twenty-fourth annual ACM-SIAM symposium on Discrete algorithms},
  title        = {$\varepsilon$-samples for kernels},
  organization = {SIAM},
  pages        = {1622--1632},
  year         = {2013},
}

@Article{owen2003quasi,
  author  = {Owen, Art B},
  title   = {Quasi-{M}onte {C}arlo sampling},
  pages   = {69--88},
  volume  = {1},
  journal = {{M}onte {C}arlo Ray Tracing: Siggraph},
  year    = {2003},
}

@Article{ferradans2014regularized,
  author    = {Ferradans, Sira and Papadakis, Nicolas and Peyr{\'e}, Gabriel and Aujol, Jean-Fran{\c{c}}ois},
  title     = {Regularized discrete optimal transport},
  number    = {3},
  pages     = {1853--1882},
  volume    = {7},
  journal   = {SIAM Journal on Imaging Sciences},
  publisher = {SIAM},
  year      = {2014},
}

@InProceedings{feder1988optimal,
  author       = {Feder, Tom{\'a}s and Greene, Daniel},
  booktitle    = {Proceedings of the twentieth annual ACM symposium on Theory of computing},
  title        = {Optimal algorithms for approximate clustering},
  organization = {ACM},
  pages        = {434--444},
  year         = {1988},
}

@Misc{meyer2015e1071,
  author = {Meyer, David and Dimitriadou, Evgenia and Hornik, Kurt and Weingessel, Andreas and Leisch, Friedrich},
  title  = {e1071: misc functions of the department of statistics, probability theory group (formerly: E1071), TU Wien. R package version 1.6-7},
  year   = {2015},
}

@Book{wu2011experiments,
  author    = {Wu, CF Jeff and Hamada, Michael S},
  title     = {Experiments: planning, analysis, and optimization},
  publisher = {John Wiley \& Sons},
  year      = {2011},
}

@Article{mosegaard1995monte,
  author    = {Mosegaard, Klaus and Tarantola, Albert},
  title     = {Monte {C}arlo sampling of solutions to inverse problems},
  number    = {B7},
  pages     = {12431--12447},
  volume    = {100},
  journal   = {Journal of Geophysical Research: Solid Earth},
  publisher = {Wiley Online Library},
  year      = {1995},
}

@Article{agueh2011barycenters,
  author    = {Agueh, Martial and Carlier, Guillaume},
  title     = {Barycenters in the {W}asserstein space},
  number    = {2},
  pages     = {904--924},
  volume    = {43},
  journal   = {SIAM Journal on Mathematical Analysis},
  publisher = {SIAM},
  year      = {2011},
}

@InProceedings{pitie2005n,
  author       = {Pitie, Francois and Kokaram, Anil C and Dahyot, Rozenn},
  booktitle    = {Computer Vision, 2005. ICCV 2005. Tenth IEEE International Conference on},
  title        = {N-dimensional probability density function transfer and its application to color transfer},
  organization = {IEEE},
  pages        = {1434--1439},
  volume       = {2},
  year         = {2005},
}

@InProceedings{li2016density,
  author    = {Li, Dangna and Yang, Kun and Wong, Wing Hung},
  booktitle = {Advances in Neural Information Processing Systems},
  title     = {Density estimation via discrepancy based adaptive sequential partition},
  pages     = {1091--1099},
  year      = {2016},
}

@InProceedings{schohn2000less,
  author       = {Schohn, Greg and Cohn, David},
  booktitle    = {ICML},
  title        = {Less is more: Active learning with support vector machines},
  organization = {Citeseer},
  pages        = {839--846},
  year         = {2000},
}

@InProceedings{krogh1995neural,
  author    = {Krogh, Anders and Vedelsby, Jesper},
  booktitle = {Advances in neural information processing systems},
  title     = {Neural network ensembles, cross validation, and active learning},
  pages     = {231--238},
  year      = {1995},
}

@Article{duong2003plug,
  author    = {Duong, Tarn and Hazelton, Martin},
  title     = {Plug-in bandwidth matrices for bivariate kernel density estimation},
  number    = {1},
  pages     = {17--30},
  volume    = {15},
  journal   = {Journal of Nonparametric Statistics},
  publisher = {Taylor \& Francis},
  year      = {2003},
}

@InProceedings{goodfellow2014generative,
  author    = {Goodfellow, Ian and Pouget-Abadie, Jean and Mirza, Mehdi and Xu, Bing and Warde-Farley, David and Ozair, Sherjil and Courville, Aaron and Bengio, Yoshua},
  booktitle = {Advances in neural information processing systems},
  title     = {Generative adversarial nets},
  pages     = {2672--2680},
  year      = {2014},
}

@Article{chung1949estimate,
  author    = {Chung, Kai-Lai},
  title     = {An estimate concerning the {K}olmogroff limit distribution},
  number    = {1},
  pages     = {36--50},
  volume    = {67},
  journal   = {Transactions of the American Mathematical Society},
  publisher = {JSTOR},
  year      = {1949},
}

@Article{meng2020sufficient,
  author  = {Meng, Cheng and Yu, Jun and Zhang, Jingyi and Ma, Ping and Zhong, Wenxuan},
  title   = {Sufficient dimension reduction for classification using principal optimal transport direction},
  volume  = {33},
  journal = {Advances in Neural Information Processing Systems},
  year    = {2020b},
}

@Book{fang1993number,
  author    = {Fang, Kai-Tai and Wang, Yuan},
  title     = {Number-theoretic methods in statistics},
  publisher = {CRC Press},
  year      = {1993},
}

@Book{leobacher2014introduction,
  author    = {Leobacher, Gunther and Pillichshammer, Friedrich},
  title     = {Introduction to quasi-{M}onte {C}arlo integration and applications},
  publisher = {Springer},
  year      = {2014},
}

@article{yu2020optimal,
  title={Optimal distributed subsampling for maximum quasi-likelihood estimators with massive data},
  author={Yu, Jun and Wang, HaiYing and Ai, Mingyao and Zhang, Huiming},
  journal={Journal of the American Statistical Association},
  pages={1--12},
  year={2020},
  publisher={Taylor \& Francis}
}

@Article{bertsekas1992auction,
  author    = {Bertsekas, Dimitri P},
  title     = {Auction algorithms for network flow problems: A tutorial introduction},
  number    = {1},
  pages     = {7--66},
  volume    = {1},
  journal   = {Computational optimization and applications},
  publisher = {Springer},
  year      = {1992},
}

@Article{kingma2013auto,
  author  = {Kingma, Diederik P and Welling, Max},
  title   = {Auto-encoding variational bayes},
  journal = {arXiv preprint arXiv:1312.6114},
  year    = {2013},
}

@Article{peyre2019computational,
  author    = {Peyr{\'e}, Gabriel and Cuturi, Marco and others},
  title     = {Computational optimal transport},
  number    = {5-6},
  pages     = {355--607},
  volume    = {11},
  journal   = {Foundations and Trends{\textregistered} in Machine Learning},
  publisher = {Now Publishers, Inc.},
  year      = {2019},
}

@InProceedings{li2012sampling,
  author       = {Li, Ninghui and Qardaji, Wahbeh and Su, Dong},
  booktitle    = {Proceedings of the 7th ACM Symposium on Information, Computer and Communications Security},
  title        = {On sampling, anonymization, and differential privacy or, k-anonymization meets differential privacy},
  organization = {ACM},
  pages        = {32--33},
  year         = {2012},
}

@Article{dasgupta2009sampling,
  author    = {Dasgupta, Anirban and Drineas, Petros and Harb, Boulos and Kumar, Ravi and Mahoney, Michael W},
  title     = {Sampling algorithms and coresets for $l_p$ regression},
  number    = {5},
  pages     = {2060--2078},
  volume    = {38},
  journal   = {SIAM Journal on Computing},
  publisher = {SIAM},
  year      = {2009},
}

@Article{gangbo1995optimal,
  author    = {Gangbo, Wilfrid and McCann, Robert J},
  title     = {Optimal maps in {M}onge's mass transport problem},
  number    = {12},
  pages     = {1653},
  volume    = {321},
  journal   = {Comptes Rendus de l'Academie des Sciences-Serie I-Mathematique},
  publisher = {Paris: Gauthier-Villars, c1984-c2001.},
  year      = {1995},
}

@Article{cohn1996active,
  author  = {Cohn, David A and Ghahramani, Zoubin and Jordan, Michael I},
  title   = {Active learning with statistical models},
  pages   = {129--145},
  volume  = {4},
  journal = {Journal of artificial intelligence research},
  year    = {1996},
}

@Article{courty2017optimal,
  author    = {Courty, Nicolas and Flamary, R{\'e}mi and Tuia, Devis and Rakotomamonjy, Alain},
  title     = {Optimal transport for domain adaptation},
  number    = {9},
  pages     = {1853--1865},
  volume    = {39},
  journal   = {IEEE transactions on pattern analysis and machine intelligence},
  publisher = {IEEE},
  year      = {2017},
}

@Article{shu2015privacy,
  author    = {Shu, Xiaokui and Yao, Danfeng and Bertino, Elisa},
  title     = {Privacy-preserving detection of sensitive data exposure},
  number    = {5},
  pages     = {1092--1103},
  volume    = {10},
  journal   = {IEEE transactions on information forensics and security},
  publisher = {IEEE},
  year      = {2015},
}

@InProceedings{meng2019large,
  author    = {Meng, Cheng and Ke, Yuan and Zhang, Jingyi and Zhang, Mengrui and Zhong, Wenxuan and Ma, Ping},
  booktitle = {Advances in Neural Information Processing Systems},
  title     = {Large-scale optimal transport map estimation using projection pursuit},
  pages     = {8116--8127},
  year      = {2019},
}

@Article{dick2013high,
  author    = {Dick, Josef and Kuo, Frances Y and Sloan, Ian H},
  title     = {High-dimensional integration: the quasi-{M}onte {C}arlo way},
  pages     = {133--288},
  volume    = {22},
  journal   = {Acta Numerica},
  publisher = {Cambridge University Press},
  year      = {2013},
}

@InProceedings{devroye1986sample,
  author       = {Devroye, Luc},
  booktitle    = {Proceedings of the 18th conference on Winter simulation},
  title        = {Sample-based non-uniform random variate generation},
  organization = {ACM},
  pages        = {260--265},
  year         = {1986},
}

@Article{ma2015statistical,
  author    = {Ma, Ping and Mahoney, Michael W and Yu, Bin},
  title     = {A statistical perspective on algorithmic leveraging},
  number    = {1},
  pages     = {861--911},
  volume    = {16},
  journal   = {The Journal of Machine Learning Research},
  publisher = {JMLR. org},
  year      = {2015},
}

@InProceedings{feldman2011scalable,
  author    = {Feldman, Dan and Faulkner, Matthew and Krause, Andreas},
  booktitle = {Advances in neural information processing systems},
  title     = {Scalable training of mixture models via coresets},
  pages     = {2142--2150},
  year      = {2011},
}

@Book{kuipers2012uniform,
  author    = {Kuipers, Lauwerens and Niederreiter, Harald},
  title     = {Uniform distribution of sequences},
  publisher = {Courier Corporation},
  year      = {2012},
}

@Article{lindsey2017optimal,
  author    = {Lindsey, Michael and Rubinstein, Yanir A},
  title     = {Optimal Transport via a {M}onge--{A}mpere Optimization Problem},
  number    = {4},
  pages     = {3073--3124},
  volume    = {49},
  journal   = {SIAM Journal on Mathematical Analysis},
  publisher = {SIAM},
  year      = {2017},
}

@Article{settles2012active,
  author    = {Settles, Burr},
  title     = {Active learning},
  number    = {1},
  pages     = {1--114},
  volume    = {6},
  journal   = {Synthesis Lectures on Artificial Intelligence and Machine Learning},
  publisher = {Morgan \& Claypool Publishers},
  year      = {2012},
}

@Article{drineas2011faster,
  author    = {Drineas, Petros and Mahoney, Michael W and Muthukrishnan, S and Sarl{\'o}s, Tam{\'a}s},
  title     = {Faster least squares approximation},
  number    = {2},
  pages     = {219--249},
  volume    = {117},
  journal   = {Numerische mathematik},
  publisher = {Springer},
  year      = {2011},
}

@Article{park2009simple,
  author    = {Park, Hae-Sang and Jun, Chi-Hyuck},
  title     = {A simple and fast algorithm for K-medoids clustering},
  number    = {2},
  pages     = {3336--3341},
  volume    = {36},
  journal   = {Expert systems with applications},
  publisher = {Elsevier},
  year      = {2009},
}

@inproceedings{cuturi2014fast,
  title={Fast computation of Wasserstein barycenters},
  author={Cuturi, Marco and Doucet, Arnaud},
  booktitle={International conference on machine learning},
  pages={685--693},
  year={2014},
  organization={PMLR}
}

@Article{meng2020more,
  author  = {Meng, Cheng and Zhang, Xinlian and Zhang, Jingyi and Zhong, Wenxuan and Ma, Ping},
  title   = {More efficient approximation of smoothing splines via space-filling basis selection},
  pages   = {723--735},
  volume  = {107},
  journal = {Biometrika},
  year    = {2020a},
}

@Article{ma2020asymptotic,
  author  = {Ma, Ping and Zhang, Xinlian and Xing, Xin and Ma, Jingyi and Mahoney, Michael W},
  title   = {Asymptotic Analysis of Sampling Estimators for Randomized Numerical Linear Algebra Algorithms},
  journal = {The 23nd International Conference on Artificial Intelligence and Statistics. 2020},
  year    = {2020},
}

@Article{drineas2006fast,
  author    = {Drineas, Petros and Kannan, Ravi and Mahoney, Michael W},
  title     = {Fast {M}onte {C}arlo algorithms for matrices {I}: Approximating matrix multiplication},
  number    = {1},
  pages     = {132--157},
  volume    = {36},
  journal   = {SIAM Journal on Computing},
  publisher = {SIAM},
  year      = {2006},
}

@InCollection{zhang2018statistical,
  author    = {Zhang, Xinlian and Xie, Rui and Ma, Ping},
  booktitle = {Handbook of Big Data Analytics},
  title     = {Statistical Leveraging Methods in Big Data},
  pages     = {51--74},
  publisher = {Springer},
  year      = {2018},
}

@Manual{transport2020,
  author = {Dominic Schuhmacher and Björn Bähre and Carsten Gottschlich and Valentin Hartmann and Florian Heinemann and Bernhard Schmitzer},
  title  = {{transport}: Computation of Optimal Transport Plans and {W}asserstein Distances},
  note   = {R package version 0.12-2},
  url    = {https://cran.r-project.org/package=transport},
  year   = {2020},
}

@Book{kaufman1987clustering,
  author    = {Kaufman, Leonard and Rousseeuw, Peter},
  title     = {Clustering by means of medoids},
  publisher = {North-Holland},
  year      = {1987},
}

@Book{lemieux2009book,
  author    = {Lemieux, Christiane},
  title     = {{M}onte {C}arlo and quasi-{M}onte {C}arlo sampling},
  publisher = {Springer, New York},
  year      = {2009},
}

@article{ai2021boptimal,
  title={Optimal subsampling for large-scale quantile regression},
  author={Ai, Mingyao and Wang, Fei and Yu, Jun and Zhang, Huiming},
  journal={Journal of Complexity},
  volume={62},
  pages={101512},
  year={2021b},
  publisher={Elsevier}
}

@InProceedings{lee2009fast,
  author    = {Lee, Dongryeol and Gray, Alexander G},
  booktitle = {Advances in Neural Information Processing Systems},
  title     = {Fast high-dimensional kernel summations using the {M}onte {C}arlo multipole method},
  pages     = {929--936},
  year      = {2009},
}

@InProceedings{xie2019online,
  author    = {Xie, Rui and Wang, Zengyan and Bai, Shuyang and Ma, Ping and Zhong, Wenxuan},
  booktitle = {The 22nd International Conference on Artificial Intelligence and Statistics},
  title     = {Online Decentralized Leverage Score Sampling for Streaming Multidimensional Time Series},
  pages     = {2301--2311},
  year      = {2019},
}

@Article{white1994continuous,
  author    = {White, Christopher A and Johnson, Benny G and Gill, Peter MW and Head-Gordon, Martin},
  title     = {The continuous fast multipole method},
  number    = {1-2},
  pages     = {8--16},
  volume    = {230},
  journal   = {Chemical physics letters},
  publisher = {Elsevier},
  year      = {1994},
}

@Article{bonneel2015sliced,
  author    = {Bonneel, Nicolas and Rabin, Julien and Peyr{\'e}, Gabriel and Pfister, Hanspeter},
  title     = {Sliced and {R}adon {W}asserstein barycenters of measures},
  number    = {1},
  pages     = {22--45},
  volume    = {51},
  journal   = {Journal of Mathematical Imaging and Vision},
  publisher = {Springer},
  year      = {2015},
}

@InProceedings{yang2003improved,
  author       = {Yang, Changjiang and Duraiswami, Ramani and Gumerov, Nail A and Davis, Larry},
  booktitle    = {null},
  title        = {Improved fast {G}auss transform and efficient kernel density estimation},
  organization = {IEEE},
  pages        = {464},
  year         = {2003},
}

\clearpage

\begin{center}
\section*{\LARGE\bf Supplemental Material}
\end{center}

\renewcommand\thesection{\Alph{section}}
\setcounter{section}{0}
\spacingset{1.5}
\section{Regularity conditions for the kernel function}

Throughout this paper, let $K(\cdot)$ be a non-negative real-valued integrable function that satisfies the following regularity conditions.

\begin{itemize}
    \item Condition 1. $K(-z)=K(z)$, for all $z\in \mathbb{R}$;
    \item Condition 2. $\int K(z)\mbox{d}z=1$;
    \item Condition 3. $\int z^2K(z)\mbox{d}z<\infty$; 
    \item Condition 4. $\int K^2(z)\mbox{d}z<\infty$;
    \item Condition 5. $\int (K'(z))^2\mbox{d}z < \infty$.
    \item Condition 6. $K(\cdot)$ is Lipschitz continuous, for all $z\in \mathbb{R}$, i.e., there exists a constant $L>0$ such that
    \[K(z_1)-K(z_2)\le L\|z_1-z_2\|_2;\]
\end{itemize}

One classical function that satisfies all these conditions is the Gaussian kernel function $K(z)=\exp\{-||z||^2/2\}/ (\int \exp\{-||z||^2/2\} \mbox{d}z$), where $||\cdot||$ denotes the Euclidean norm.
We refer to \cite{scott2015multivariate} for more discussion on different choices of kernel functions that satisfy these regularity conditions.

\section{Essential lemmas}
The following lemmas are essential to the proof. The proof of the first three lemmas can be found in \cite{kuipers2012uniform,gangbo1995optimal}, and \cite{lindsey2017optimal}, respectively.
The proof of Lemma~S4 can be found from Theorem~1 and some remarks on page 1362 in \cite{trillos2014rate}.
The proof of the last lemma is provided below.

\begin{lemma}\label{lem_1}
\textit{(Koksma-Hlawka inequality)} Denote $\S_r=\{s_1,...,s_r\}$ as a set of data points in $[0,1]^d$ and $f$ is a function on $[0,1]^d$ with bounded total variation $\mathcal{V}(f)$. The total variation is defined in the sense of Hardy and Krause \citep{owen2003quasi}. Then, 
\begin{eqnarray*}
\left|\int_{[0,1]^d}f(x)\mbox{d}x-\frac{1}{r}\sum_{i=1}^rf(s_i)\right|\leq D^*(\S_r)\mathcal{V}(f_{}).
\end{eqnarray*}
\end{lemma}

\begin{lemma}\label{lem_2}
\textit{(Existence and uniqueness of the optimal transport map)} Let the transportation cost be a strictly convex function, and $f_X,f_Y$ be the probability density functions with bounded support. The optimal transport map $\phi^*$ that minimizes the transportation cost is unique and is a one to one map.
\end{lemma}

\begin{lemma}\label{lem_3}
\textit{(Differentiable of the optimal transport map)} Let $\Omega$ and $\Lambda$ be bounded open sets in $\mathbb{R}^d$ with $\Lambda$ convex, and let $f_X$ and $f_Y$ be probability density functions on $\Omega$ and $\Lambda$, respectively, each bounded away from zero and infinity. Assume that $f_X$ and $f_Y$ are in $C^{0,\alpha}(\bar\Omega)$ and $C^{0,\alpha}(\bar\Lambda)$, repsectively. Then there exists a unique solution of the corresponding Monge problem for the quadratic cost, i.e.,
\begin{eqnarray*}
\min_{\{\phi:\Omega\to\Lambda:\phi_{\#}(p_X)=p_Y\}}\int_\Omega \|X-\phi(X)\|^2dp_X,
\end{eqnarray*}
and, moreover, $\phi$ is in $C^{1,\alpha}(\Omega),$ where $C^{k,\alpha}(\Omega)$ is consisted  by the functions on $\Omega$ having continuous derivatives up to order $k$ and such that the $k$th partial derivatives are Holder continuous with exponent $\alpha.$
\end{lemma}

\begin{lemma}\label{lem_5}\textit{(\cite{trillos2014rate})} Let $D\subseteq \mathbb{R}^d$ be a bounded, connected, open set with Lipschitz boundary. Let $\nu$ be a probability measure on $D$ with density $p: D\to(0,\infty)$ such that there exsits $C_1\ge 1$ for which $C_1^{-1}\le p(x)\le C_1 (\forall x\in D).$ Let $X_1,\ldots,X_n$ be i.i.d. sample from $\nu.$ Consider $\nu_n$ the empirical measure $\nu_n=n^{-1}\sum_{i=1}^n\delta_{X_i}.$ Then, for any fixed $\alpha>2,$ except on a set with probability $O(n^{-\alpha/2}),$
\begin{equation}
    W_\infty(\nu,\nu_n)\le C_2\begin{cases} \frac{\log(n)^{3/4}}{n^{1/2}} & d=2,\\
    \frac{{n}^{1/d}}{n^{1/d}} & d\ge 3,
    \end{cases}
\end{equation}
where $W_\infty$ is the $\infty$-transportation distance, $C_2$ is a constant depends on $\alpha,C_1,D$ only. 

Moreover, there exist some transportation map $T_n$ between $\nu$ and $\nu_n,$ such that 
\begin{equation}\label{eq:lems4}
    \|T_n-I_d\|_{L_\infty(D)}\le C_2\begin{cases} \frac{\log(n)^{3/4}}{n^{1/2}} & d=2,\\
    \frac{{\log(n)}^{1/d}}{n^{1/d}} & d\ge 3,
    \end{cases}
\end{equation}
holds,
where $\|\cdot\|_{L_\infty(D)}$ denotes the $L_\infty$ norm on $D$ and $I_d$ is the identity map.
\end{lemma}

\begin{lemma}\label{lem_4}
Let $C_1$ and $L$ be positive constants.
For Lipschitz continuous functions $f_j(z_j), j=1,\ldots,d$ with $\sup_{z_j}|f_j(z_j)|\le C_1$, we have
\begin{eqnarray*}
\left|\prod_{j=1}^df_j(z_j)- \prod_{j=1}^df_j(z'_j) \right|\le C_1^{d-1} \sum_{j=1}^d L|z_j-z'_j|.
\end{eqnarray*}
\end{lemma}

\begin{proof}[of Lemma S5]
\begin{align}
   &\left|\prod_{j=1}^df_j(z_i)- \prod_{j=1}^df_j(z'_j) \right |\nonumber\\
   \le& \left|f_1(z_1)\prod_{j=2}^df_j(z_j)- f_1(z_1)\prod_{j=2}^df_j(z'_j)\right |+\left|f_1(z_1)\prod_{j=2}^df_j(z'_j)-f_1(z'_1)\prod_{j=2}^df_j(z'_j)\right |\label{eq:s51}\\
   \le& C_1\left|\prod_{j=2}^df_j(z_j)- \prod_{j=2}^df_j(z'_j)\right |+ C_1^{d-1}\left|f_1(z_1)-f_1(z'_1)\right |\label{eq:s61}\\
   \le & C_1^{d-1}\left(\sum_{j=1}^d \left|f_j(z_j)- f_j(z'_j)\right |\right )\label{eq:s71}\\
   \le & C_1^{d-1} \sum_{j=1}^d L|z_j-z'_j|. \nonumber
\end{align}
The (\ref{eq:s71}) holds by using the same technique in (\ref{eq:s51}) and (\ref{eq:s61}), recursively.
\end{proof}

\section{Proof of Theorem 1}

For any fixed point $z \in \mathbb{R}^d$, the full sample estimator can be written as
\begin{eqnarray}\label{e1}
\pzx = \frac{1}{n}\sum_{i=1}^n \left\{\prod_{j=1}^d K_h\left(z_j-x_{ij}\right)\right\}.
\end{eqnarray}
Let $X$ be the random variable with probability distribution function $p$.
Lemma~\ref{lem_2} indicates there exists an optimal transport map $\phi^*$ such that $\phi^*(X)$ follows the uniform distribution on $[0,1]^d$, i.e., $U[0,1]^d$.
Lemma~\ref{lem_2} also indicates $\phi^*$ is a one-to-one map, and thus the map $(\phi^*)^{-1}$ is well-defined.
One thus can calculate the expectation of Equation~(\ref{e1}) using
\begin{eqnarray}\label{e2}
\mbox{E}(\pzx)= \int_{[0,1]^d}g_{z}(u)du,
\end{eqnarray}
where $g_{z}(u)= \prod_{j=1}^d K_h\left(z_j-((\phi^*)^{-1}(u))_j\right)$.

Recall that the optimal transport map $\widehat{\phi}$ is a one-to-one map from $\{x_i\}_{i=1}^n$ to a uniformly-distributed sample $\{u_i\}_{i=1}^n$, and thus its inverse map $\widehat{\phi}^{-1}$ is well-defined on $\{u_i\}_{i=1}^n$.
Following the notations in Algorithm~1, for $i=1,\ldots,r$, we can write the selected data point $x^*_i$ as $\widehat{\phi}^{-1}(u_i^*)$.
Consequently, the proposed subsample estimator can be written as
\begin{align}
\widehat{p}_{\rm est.}(z) &=\frac{1}{r}\sum_{i=1}^r \left\{\prod_{j=1}^d K_h\left(z_j-(\widehat\phi^{-1}(u_i^*))_j\right)\right\}\nonumber\\
&=\frac{1}{r}\sum_{i=1}^r g_{z,{\rm est.}}(u_i^*), \label{e33}
\end{align}
where $g_{z,{\rm est.}}(u)= \prod_{j=1}^d K_h\left(z_j-(\widehat\phi^{-1}(u))_j\right),$ for $u\in \{u_i\}_{i=1}^n.$

Let
\begin{align}
\pzxr 
&=\frac{1}{r}\sum_{i=1}^r g_{z}(u_i^*)\label{e3}.
\end{align}

The MSE of the proposed estimator, i.e., $\mbox{MSE}(\widehat{p}_{\rm est.}(z))$, can be bounded as follows,
\begin{align}
\mbox{MSE}(\widehat{p}_{\rm est.}(z)) & = \mbox{E}\Big(\widehat{p}_{\rm est.}(z)-p(z)\Big)^2\nonumber\\
& = \mbox{E}\Big(\widehat{p}_{\rm est.}(z)-\pzxr+\pzxr-p(z)\Big)^2\nonumber\\
&\le 2\mbox{E}\Big(\widehat{p}_{\rm est.}(z)-\pzxr\Big)^2+2\mbox{E}\Big(\pzxr-p(z)\Big)^2\nonumber\\
& = 2\mbox{E}\Big(\widehat{p}_{\rm est.}(z)-\pzxr\Big)^2+\mbox{E}\Big(\pzxr-\mbox{E}(\pzx)+\mbox{E}(\pzx)-p(z)\Big)^2\nonumber\\
& \leq 2\mbox{E}\Big|\widehat{p}_{\rm est.}(z)-\pzxr\Big|^2+2\mbox{E}\Big|\pzxr-\mbox{E}(\pzx) \Big|^2+2\mbox{E}\Big|\mbox{E}(\pzx)-p(z) \Big|^2.\label{eq:s92}
\end{align}

It is known that under Conditions~(a) and (b), 
\begin{eqnarray} \label{e9}
\mbox{E}\Big|\mbox{E}(\pzx)-p(z) \Big|^2=O(h^4),
\end{eqnarray}
see \citet{scott2015multivariate} for more details.
In the following, we derive the upper bound for the first and the second term of the right-hand-side of Inequality~(\ref{eq:s92}), respectively.

We first show that, under Conditions 4 and 6, we have \begin{eqnarray}\label{s12}
K_h\left(z_j-((\phi^*)^{-1}(u))_j\right)\le C_1 \quad\mbox{   and   }\quad K_h\left(z_j-(\widehat\phi^{-1}(u))_j\right)\le C_1
\end{eqnarray}
for some positive constant $C_1$, $j=1,\ldots, d$
This is because, if there exists a $z_j$ and $u$ such that $K_h\left(z_j-((\phi^*)^{-1}(u))_j\right)= \infty$; then Condition~6 indicates one can find a non empty set $\mathcal{S}$, such that $K_h\left(z_j-((\phi^*)^{-1}(u^+))_j\right)= \infty$ for any $z_j\in \mathcal{S}$. 
Consequently, we have $\int_\mathcal{S} K_h^2\left(z_j-((\phi^*)^{-1}(u^+))_j\right)\mbox{d}z_j= \infty,$ which leads to a contradiction.

Using Inequalities~(\ref{s12}), Condition 6, and Lemma~\ref{lem_4}, we have
\begin{align}
 |g_z(u)-g_{z,{\rm est.}}(u)|&=\left|\prod_{j=1}^d K_h\left(z_j-((\phi^*)^{-1}(u))_j\right)-\prod_{j=1}^d K_h\left(z_j-(\widehat\phi^{-1}(u))_j\right) \right |\nonumber\\
    &\le C_1^{d-1}\sum_{j=1}^d L\|((\phi^*)^{-1}(u))_j-(\widehat\phi^{-1}(u))_j\|_2\nonumber\\
    &=C_1^{d-1}L\|((\phi^*)^{-1}(u))-(\widehat\phi^{-1}(u))\|_1,   \label{eq:s4}
\end{align}
where $\|\cdot\|_2$ and $\|\cdot\|_1$ are the $l_2$ norm and $l_1$ norm, respectively.

Combining Equations~(\ref{e33}),(\ref{e3}) and (\ref{eq:s4}), for $d\ge 3$, we have
\begin{align}
    |\pzxr-\widehat{p}_{\rm est.}(z)|&\le\frac{1}{r}\sum_{i=1}^r|g_z(u_i^*)-g_{z,{\rm est.}}(u_i^*)|\nonumber\\
    &\le C_1^{d-1}L\sup_{u\in\{u_i\}_{i=1}^n}\|((\phi^*)^{-1}(u))-((\widehat\phi)^{-1}(u))\|_1\nonumber\\
    &\le C_1^{d-1}L{d}\sup_{u\in\{u_i\}_{i=1}^n}\|((\phi^*)^{-1}(u))-((\widehat\phi)^{-1}(u))\|_\infty\nonumber\\
    &=O_p\left(\frac{\log(n)^{1/d}}{n^{1/d}}\right) \label{eq:s8}\\
    &=O_p\left(\frac{\log(n)^{1/d}}{\log(r)^{1/d}}\right)O_p\left(\frac{\log(r)^{1/d}}{r}\right)O_p\left(\frac{r}{n^{1/d}}\right)\nonumber\\
    &=O_p\left(\frac{\log(r)^{1/d}}{r}\right),\label{eq:s9}
\end{align}
where $\|\cdot\|_\infty$ is the $l_\infty$.
Here, Equation~(\ref{eq:s8}) comes from (\ref{eq:lems4}) in Lemma \ref{lem_5}, and Equation~(\ref{eq:s9}) comes from the assumption that $r=O(n^{1/d})$.
For the case when $d=2$, according to Lemma~\ref{lem_5}, we have
\begin{align}
    |\pzxr-\widehat{p}_{\rm est.}(z)|
    &=O_p\left(\frac{\log(n)^{1/d+1/4}}{n^{1/d}}\right) =O_p\left(\frac{\log(r)^{1/d+1/4}}{r}\right).\label{eq:s15}
\end{align}
Combining Equations~(\ref{eq:s9}) and (\ref{eq:s15}), for $d\ge 2$, we have
\begin{equation}\label{eq:s16}
    \mbox{E}|\pzxr-\widehat{p}_{\rm est.}(z)|
     =O\left(\frac{\log(r)^{1/d+1/4}}{r}\right).
\end{equation}

Next, we consider the upper bound for $(\pzxr-\mbox{E}\left(\pzx\right))^2$.
Combining the results in Equations~(\ref{e2}), (\ref{e3}) and Lemma~\ref{lem_1}, we have,
\begin{align}
\left|\pzxr-\mbox{E}\left(\pzx\right) \right|
& = \left|\frac{1}{r}\sum_{i=1}^r g_{z}(u_i^*)-\int_{[0,1]^d}g_{z}(u)du\right| \leq D^*(\U_r^*)\mathcal{V}(g_{z})\label{e0}.
\end{align}

Following the definition of the total variation, we have
\[
\mathcal{V}(g_{z}) = \int_{[0,1]^d}\|\bigtriangledown g_{z}(u)\|du,
\]
where $\|\cdot\|$ is the $l_2$ norm, and $\bigtriangledown g_{z}(u) = \left(\frac{\partial g_{z}(u)}{\partial u_1},..., \frac{\partial g_{z}(u)}{\partial u_d}\right)^T$. 
To simplify the expression of $g_z(u)$, we let
\[
\mathcal{K}(x) = \prod_{j=1}^d K(x_j), \quad x \in \mathbb{R}^d.
\]
One thus has $g_z(u)=\frac{1}{h^d}\mathcal{K}\left(\frac{z-(\phi^*)^{-1}(u)}{h}\right)$. Let $\omega = \frac{z-(\phi^*)^{-1}(u)}{h}$, we have
\[
\bigtriangledown g_{z}(u) = \frac{1}{h^d}J_{\omega \to u}^T\bigtriangledown \mathcal{K}(\omega),
\]
where 
\[
J_{\omega \to u}=
\begin{bmatrix}
    \frac{\partial \omega_1}{\partial u_1} & \dots  & \frac{\partial \omega_1}{\partial u_d} \\
    \vdots & \ddots & \vdots \\
    \frac{\partial \omega_d}{\partial u_1} & \dots  & \frac{\partial \omega_d}{\partial u_d}
\end{bmatrix}.
\]
Similarly, we define
\[
J_{u \to \omega}=
\begin{bmatrix}
    \frac{\partial u_1}{\partial \omega_1} & \dots  & \frac{\partial u_1}{\partial \omega_d} \\
    \vdots & \ddots & \vdots \\
    \frac{\partial u_d}{\partial \omega_1} & \dots  & \frac{\partial u_d}{\partial \omega_d}
\end{bmatrix},
J_{\phi^*}=
\begin{bmatrix}
    \frac{\partial (\phi^*(x))_1}{\partial x_1} & \dots  & \frac{\partial (\phi^*(x))_1}{\partial x_d} \\
    \vdots & \ddots & \vdots \\
    \frac{\partial (\phi^*(x))_d}{\partial x_1} & \dots  & \frac{\partial (\phi^*(x))_d}{\partial x_d}
\end{bmatrix}, 
\]
and
\[
J_{(\phi^*)^{-1}}=
\begin{bmatrix}
    \frac{\partial ((\phi^*)^{-1}(u))_1}{\partial u_1} & \dots  & \frac{\partial ((\phi^*)^{-1}(u))_1}{\partial u_d} \\
    \vdots & \ddots & \vdots \\
    \frac{\partial ((\phi^*)^{-1}(u))_d}{\partial u_1} & \dots  & \frac{\partial ((\phi^*)^{-1}(u))_d}{\partial u_d}
\end{bmatrix}.
\]

Notice that $J_{\omega \to u} = -\frac{1}{h}J_{(\phi^*)^{-1}}$, one thus has
\[
\bigtriangledown g_{z}(u) = \frac{1}{h^{d+1}}J_{(\phi^*)^{-1}}^T\bigtriangledown \mathcal{K}(\omega).
\]
Using the Jensen's inequality, we have
\begin{align}
\mathcal{V}^2(g_{z}) & \leq\int_{[0,1]^d}||\bigtriangledown g_{z}(u)||^2du\nonumber\\
& = \int_{[0,1]^d}(\bigtriangledown g_{z}(u))^T\bigtriangledown g_{z}(u)du\nonumber\\
& = \frac{1}{h^{2d+2}}\int_{[0,1]^d}(\bigtriangledown \mathcal{K}(\omega))^TJ_{(\phi^*)^{-1}}J_{(\phi^*)^{-1}}^T\bigtriangledown \mathcal{K}(\omega)du\nonumber\\
& = \frac{1}{h^{2d+2}}\int_{\Omega}(\bigtriangledown \mathcal{K}(\omega))^TJ_{(\phi^*)^{-1}}J_{(\phi^*)^{-1}}^T\bigtriangledown \mathcal{K}(\omega)|\det(J_{u \to \omega})|d\omega\nonumber\\
& = \frac{1}{h^{d+2}}\int_{\Omega}(\bigtriangledown \mathcal{K}(\omega))^TJ_{(\phi^*)^{-1}}J_{(\phi^*)^{-1}}^T\bigtriangledown \mathcal{K}(\omega)|\det(J_{\phi^*})|d\omega\label{e4},
\end{align}
where the fact that $u = \phi^*(z-h\omega)$, $J_{u \to \omega} = -hJ_{\phi^*}$, and $|\det(J_{u \to \omega})| = h^d|\det(J_{\phi^*})|$ are used in the last equation.

Notice that 
\begin{align}
(\bigtriangledown \mathcal{K}(\omega))^TJ_{(\phi^*)^{-1}}J_{(\phi^*)^{-1}}^T\bigtriangledown \mathcal{K}(\omega) & = \mathrm{tr}\left((\bigtriangledown \mathcal{K}(\omega))^TJ_{(\phi^*)^{-1}}J_{(\phi^*)^{-1}}^T\bigtriangledown \mathcal{K}(\omega)\right)\nonumber\\
& = \mathrm{tr}\left(\bigtriangledown \mathcal{K}(\omega)(\bigtriangledown \mathcal{K}(\omega))^TJ_{(\phi^*)^{-1}}J_{(\phi^*)^{-1}}^T\right)\nonumber\\
& \leq \mathrm{tr}\left(\bigtriangledown \mathcal{K}(\omega)(\bigtriangledown \mathcal{K}(\omega))^T\right)\mathrm{tr}\left(J_{(\phi^*)^{-1}}J_{(\phi^*)^{-1}}^T\right)\label{e5}.
\end{align}

For the first term in the right-hand-side of Inequality~(\ref{e5}), i.e., $\mathrm{tr}\left(\bigtriangledown \mathcal{K}(\omega)(\bigtriangledown \mathcal{K}(\omega))^T\right)$, we have
\begin{align}
\mathrm{tr}\left(\bigtriangledown \mathcal{K}(\omega)(\bigtriangledown \mathcal{K}(\omega))^T\right) & = \mathrm{tr}\left((\bigtriangledown \mathcal{K}(\omega))^T\bigtriangledown \mathcal{K}(\omega)\right)\nonumber\\
& = (\bigtriangledown \mathcal{K}(\omega))^T\bigtriangledown \mathcal{K}(\omega)\nonumber\\
& = \sum_{k=1}^d\left(\left\{\prod_{j\neq k}K^2(\omega_j)\right\}\left(K'(\omega_k)\right)^2\right)\label{e66}.
\end{align}

For the second term in the right-hand-side of Inequality~(\ref{e5}), we have
\begin{eqnarray}\label{e6}
\mathrm{tr}\left(J_{(\phi^*)^{-1}}J_{(\phi^*)^{-1}}^T\right)\leq C,
\end{eqnarray}
for a positive constant $C$.
This is because $(\phi^*)^{-1}$ is an optimal transport map that defined on a bounded domain $[0,1]^d$.
Furthermore, according to Lemma S3, the derivative of $(\phi^*)^{-1}$ is continuous.
Consequently, all the entries in $J_{(\phi^*)^{-1}}$ are finite, and thus Inequality~(\ref{e6}) can be satisfied.
Plugging Equation~(\ref{e66}) and Inequality~(\ref{e5}) back into Equation~(\ref{e4}), we have
\begin{align}
\mathcal{V}^2(g_{z}) & \leq \frac{1}{h^{d+2}}C\int\cdots\int\sum_{k=1}^d\left(\left\{\prod_{j\neq k}K^2(\omega_j)\right\}\left(K'(\omega_k)\right)^2\right)d\omega_1\cdots d\omega_d\nonumber\\
& = \frac{1}{h^{d+2}}C\sum_{k=1}^d\left\{\prod_{j\neq k}\int_{\Omega_j}K^2(\omega_j)d\omega_j\int_{\Omega_k}\left(K'(\omega_k)\right)^2d\omega_k\right\}\nonumber\\
& = O\left(\frac{1}{h^{d+2}}\right)\label{e7}.
\end{align}

Combining Inequalities~(\ref{e7}) and ~(\ref{e0}), we have 
\begin{eqnarray}\label{e8}
\mbox{E}\Big(\pzxr-\mbox{E}\left(\pzx\right) \Big)^2 \leq \Big(D^*(\U_r^*)\Big)^2\mathcal{V}^2(g_{z}) = O\left(\frac{1}{r^{2(1-\delta)} h^{d+2}}\right).
\end{eqnarray}


Plugging (\ref{eq:s16}), (\ref{e8}) and (\ref{e9}) into (\ref{eq:s92}) yields  
\begin{eqnarray*}
\mbox{MSE}(\pzxr) &=&  O\left(\frac{\log(r)^{2/d+1/2}}{r^2}\right)+ O{\left(\frac{1}{r^{2(1-\delta)} h^{d+2}}\right)} + O(h^4)\\
&=&O{\left(\frac{1}{r^{2(1-\delta)} h^{d+2}}\right)} + O(h^4).
\end{eqnarray*}
Consequently, when $h=O(r^{-\frac{2(1-\delta)}{6+d}})$, we have
\begin{equation*}
\mbox{MSE}(\pzxr) = O(r^{-\frac{8(1-\delta)}{6+d}}).
\end{equation*}


\end{document}


\def\spacingset#1{\renewcommand{\baselinestretch}%
{#1}\small\normalsize} \spacingset{1}

\newcommand{\red}[1]{\textcolor{red}{#1}}
\newcommand{\blue}[1]{\textcolor{blue}{#1}}

\addtolength{\textheight}{.5in}%

\if1\blind
{
  \title{\bf Supplementary materials for ``An optimal transport approach for selecting a representative subsample with application in efficient kernel density estimation"}
  \date{}
  \author{Jingyi Zhang \\
    Center for Statistical Science, Tsinghua University\\
    \\
    Cheng Meng \footnote{Joint first author} \\
    Center for Applied Statistics, \\ Institute of Statistics and Big Data, Renmin University of China\\
    \\
    Jun Yu \\
    School of Mathematics and Statistics, Beijing Institute of Technology
    \\
    \\
    Mengrui Zhang, Wenxuan Zhong, and Ping Ma\footnote{Corresponding author}\\
    Department of Statistics, University of Georgia.
    }
  \maketitle
} \fi

\spacingset{1.5}

\begin{description}

\item[Title:] Proofs of theoretical results and related materials.

\item[Regularity conditions for the kernel function:] We present some essential regularity conditions. 

\item[Essential lemmas:] We introduce some essential lemmas and their proofs.

\item[Proof of Theorem 1:] A proof of Theorem~1.

\item[CPU time comparison:] We provide simulation results to compare the CPU time of the proposed method with the competitors.

\end{description}

\section{Regularity conditions for the kernel function}

Throughout this paper, let $K(\cdot)$ be a non-negative real-valued integrable function that satisfies the following regularity conditions.

\begin{itemize}
    \item Condition 1. $K(-z)=K(z)$, for all $z\in \mathbb{R}$;
    \item Condition 2. $\int K(z)\mbox{d}z=1$;
    \item Condition 3. $\int z^2K(z)\mbox{d}z<\infty$; 
    \item Condition 4. $\int K^2(z)\mbox{d}z<\infty$;
    \item Condition 5. $\int (K'(z))^2\mbox{d}z < \infty$.
    \item Condition 6. $K(\cdot)$ is Lipschitz continuous, for all $z\in \mathbb{R}$, i.e., there exists a constant $L>0$ such that
    \[K(z_1)-K(z_2)\le L|z_1-z_2|;\]
\end{itemize}

One classical function that satisfies all these conditions is the Gaussian kernel function $K(z)=\exp\{-z^2/2\}/ (\int \exp\{-z^2/2\} \mbox{d}z$).
We refer to \cite{scott2015multivariate} for more discussion on different choices of kernel functions that satisfy these regularity conditions.

\section{Essential lemmas}
The following lemmas are essential to the proof. The proof of the first three lemmas can be found in \cite{kuipers2012uniform,gangbo1995optimal}, and \cite{lindsey2017optimal}, respectively.
The proof of Lemma~4 and 5 can be found in \cite{trillos2014rate}.
The proof of the last lemma is provided below.

\begin{lemma}\label{lem_1}
\textit{(Koksma-Hlawka inequality)} Denote $\S_r=\{\bs_1,...,\bs_r\}$ as a set of data points in $[0,1]^d$ and $f$ is a function on $[0,1]^d$ with bounded total variation $\mathcal{V}(f)$. The total variation is defined in the sense of Hardy and Krause \citep{owen2003quasi}. Then, 
\begin{eqnarray*}
\left|\int_{[0,1]^d}f(\x)\mbox{d}\x-\frac{1}{r}\sum_{i=1}^rf(\bs_i)\right|\leq D^*(\S_r)\mathcal{V}(f_{}).
\end{eqnarray*}
\end{lemma}

\begin{lemma}\label{lem_2}
\textit{(Existence and uniqueness of the optimal transport map)} Let the transportation cost be a strictly convex function, and $f_X,f_Y$ be the probability density functions with bounded support. The optimal transport map $\phi^*$ that minimizes the transportation cost is unique and is a one to one map.
\end{lemma}

\begin{lemma}\label{lem_3}
\textit{(Differentiability of the optimal transport map)} Let $\Omega$ and $\Lambda$ be bounded open sets in $\mathbb{R}^d$ with $\Lambda$ convex, and let $f_X$ and $f_Y$ be probability density functions on $\Omega$ and $\Lambda$, respectively, each bounded away from zero and infinity. Assume that $f_X$ and $f_Y$ are in $C^{0,\alpha}(\bar\Omega)$ and $C^{0,\alpha}(\bar\Lambda)$, respectively. Then there exists a unique solution of the corresponding Monge problem for the quadratic cost, i.e.,
\begin{eqnarray*}
\min_{\{\phi:\Omega\to\Lambda:\phi_{\#}(p_X)=p_Y\}}\int_\Omega \|X-\phi(X)\|^2dp_X,
\end{eqnarray*}
and, moreover, $\phi$ is in $C^{1,\alpha}(\Omega),$ where $C^{k,\alpha}(\Omega)$ is consisted  by the functions on $\Omega$ having continuous derivatives up to order $k$ and such that the $k$th partial derivatives are H$\ddot{o}$lder continuous with exponent $\alpha.$
\end{lemma}

\begin{lemma}\label{lem_5}\textit{(\cite{trillos2014rate})} Let $D\subseteq \mathbb{R}^d$ be a bounded, connected, open set with Lipschitz boundary. Let $\nu$ be a probability measure on $D$ with density $p: D\to(0,\infty)$ such that there exists $C_1\ge 1$ for which $C_1^{-1}\le p(\x)\le C_1 (\forall \x\in D).$ Let $\{\x_i\}_{i=1}^n$ be i.i.d. sample from $\nu.$ Consider $\nu_n$ the empirical measure $\nu_n=n^{-1}\sum_{i=1}^n\delta_{\x_i}.$ Then, for any fixed $\alpha>2,$ except on a set with probability $O(n^{-\alpha/2}),$
\begin{equation}
    W_\infty(\nu,\nu_n)\le C_2\begin{cases} \frac{\log(n)^{3/4}}{n^{1/2}} & d=2,\\
    \frac{{n}^{1/d}}{n^{1/d}} & d\ge 3,
    \end{cases}
\end{equation}
where $W_\infty$ is the $\infty$-transportation distance, $C_2$ is a constant depends on $\alpha,C_1,D$ only. 

Moreover, there exist some transportation map $T_n$ between $\nu$ and $\nu_n,$ such that 
\begin{equation}\label{eq:lems4}
    \|T_n-I_d\|_{L_\infty(D)}\le C_2\begin{cases} \frac{\log(n)^{3/4}}{n^{1/2}} & d=2,\\
    \frac{{\log(n)}^{1/d}}{n^{1/d}} & d\ge 3,
    \end{cases}
\end{equation}
holds,
where $\|\cdot\|_{L_\infty(D)}$ denotes the $L_\infty$ norm on $D$ and $I_d$ is the identity map.
\end{lemma}

\begin{lemma}\label{lem_4}
Let $C_1$ and $L$ be positive constants.
For Lipschitz continuous functions $f_j(z_j), j=1,\ldots,d$ with $\sup_{z_j}|f_j(z_j)|\le C_1$, we have
\begin{eqnarray*}
\left|\prod_{j=1}^df_j(z_j)- \prod_{j=1}^df_j(z'_j) \right|\le C_1^{d-1} \sum_{j=1}^d L|z_j-z'_j|.
\end{eqnarray*}
\end{lemma}

\begin{proof}[of Lemma 6]
\begin{align}
   &\left|\prod_{j=1}^df_j(z_i)- \prod_{j=1}^df_j(z'_j) \right |\nonumber\\
   \le& \left|f_1(z_1)\prod_{j=2}^df_j(z_j)- f_1(z_1)\prod_{j=2}^df_j(z'_j)\right |+\left|f_1(z_1)\prod_{j=2}^df_j(z'_j)-f_1(z'_1)\prod_{j=2}^df_j(z'_j)\right |\label{eq:s51}\\
   \le& C_1\left|\prod_{j=2}^df_j(z_j)- \prod_{j=2}^df_j(z'_j)\right |+ C_1^{d-1}\left|f_1(z_1)-f_1(z'_1)\right |\label{eq:s61}\\
   \le & C_1^{d-1}\left(\sum_{j=1}^d \left|f_j(z_j)- f_j(z'_j)\right |\right )\label{eq:s71}\\
   \le & C_1^{d-1} \sum_{j=1}^d L|z_j-z'_j|. \nonumber
\end{align}
The (\ref{eq:s71}) holds by using the same technique in (\ref{eq:s51}) and (\ref{eq:s61}), recursively.
\end{proof}

\section{Proof of Theorem 1}

For any fixed point $z \in \mathbb{R}^d$, the full sample estimator can be written as
\begin{eqnarray}\label{e1}
\pzx = \frac{1}{n}\sum_{i=1}^n \left\{\prod_{j=1}^d K_h\left(z_j-x_{ij}\right)\right\}.
\end{eqnarray}
Let $X$ be the random variable with probability distribution function $p$.
Lemma~\ref{lem_2} indicates there exists an optimal transport map $\phi^*$ such that $\phi^*(X)$ follows the uniform distribution on $[0,1]^d$, i.e., $U[0,1]^d$.
Lemma~\ref{lem_2} also indicates $\phi^*$ is a one-to-one map, and thus the map $(\phi^*)^{-1}$ is well-defined.
One thus can calculate the expectation of Equation~(\ref{e1}) using
\begin{eqnarray}\label{e2}
\mbox{E}(\pzx)= \int_{[0,1]^d}g_{z}(\bu)\mbox{d}\bu,
\end{eqnarray}
where $g_{z}(\bu)= \prod_{j=1}^d K_h\left(z_j-((\phi^*)^{-1}(\bu))_j\right)$.

Recall that the optimal transport map $\widehat{\phi}$ is a one-to-one map from $\{\x_i\}_{i=1}^n$ to a uniformly-distributed sample $\{\bu_i\}_{i=1}^n$, and thus its inverse map $\widehat{\phi}^{-1}$ is well-defined on $\{\bu_i\}_{i=1}^n$.
Following the notations in Algorithm~1, for $i=1,\ldots,r$, we can write the selected data point $\x^*_i$ as $\widehat{\phi}^{-1}(\bu_i^*)$.
Consequently, the proposed subsample estimator can be written as
\begin{align}
\widehat{p}_{\rm est.}(\bz) &=\frac{1}{r}\sum_{i=1}^r \left\{\prod_{j=1}^d K_h\left(z_j-(\widehat\phi^{-1}(\bu_i^*))_j\right)\right\}\nonumber\\
&=\frac{1}{r}\sum_{i=1}^r g_{z,{\rm est.}}(\bu_i^*), \label{e33}
\end{align}
where $g_{z,{\rm est.}}(\bu)= \prod_{j=1}^d K_h\left(z_j-(\widehat\phi^{-1}(\bu))_j\right),$ for $\bu\in \{\bu_i\}_{i=1}^n.$

Let
\begin{align}
\pzxr 
&=\frac{1}{r}\sum_{i=1}^r g_{z}(\bu_i^*)\label{e3}.
\end{align}

The MSE of the proposed estimator, i.e., $\mbox{MSE}(\widehat{p}_{\rm est.}(\bz))$, can be bounded as follows,
\begin{align}
\mbox{MSE}(\widehat{p}_{\rm est.}(\bz)) & = \mbox{E}\Big(\widehat{p}_{\rm est.}(\bz)-p(\bz)\Big)^2\nonumber\\
& = \mbox{E}\Big(\widehat{p}_{\rm est.}(\bz)-\pzxr+\pzxr-p(\bz)\Big)^2\nonumber\\
&\le 2\mbox{E}\Big(\widehat{p}_{\rm est.}(\bz)-\pzxr\Big)^2+2\mbox{E}\Big(\pzxr-p(\bz)\Big)^2\nonumber\\
& = 2\mbox{E}\Big(\widehat{p}_{\rm est.}(\bz)-\pzxr\Big)^2+\mbox{E}\Big(\pzxr-\mbox{E}(\pzx)+\mbox{E}(\pzx)-p(\bz)\Big)^2\nonumber\\
& \leq 2\mbox{E}\Big|\widehat{p}_{\rm est.}(\bz)-\pzxr\Big|^2+2\mbox{E}\Big|\pzxr-\mbox{E}(\pzx) \Big|^2+2\mbox{E}\Big|\mbox{E}(\pzx)-p(\bz) \Big|^2.\label{eq:s92}
\end{align}

It is known that under Conditions~(a) and (b), 
\begin{eqnarray} \label{e9}
\mbox{E}\Big|\mbox{E}(\pzx)-p(\bz) \Big|^2=O(h^4),
\end{eqnarray}
see \citet{scott2015multivariate} for more details.
In the following, we derive the upper bound for the first and the second term of the right-hand-side of Inequality~(\ref{eq:s92}), respectively.

We first show that, under Conditions 4 and 6, we have \begin{eqnarray}\label{s12}
K_h\left(z_j-((\phi^*)^{-1}(\bu))_j\right)\le C_1 \quad\mbox{   and   }\quad K_h\left(z_j-(\widehat\phi^{-1}(\bu))_j\right)\le C_1
\end{eqnarray}
for some positive constant $C_1$, $j=1,\ldots, d$
This is because, if there exists a $z_j$ and $\bu$ such that $K_h\left(z_j-((\phi^*)^{-1}(\bu))_j\right)= \infty$; then Condition~6 indicates one can find a non empty set $\mathcal{S}$, such that $K_h\left(z_j-((\phi^*)^{-1}(\bu^+))_j\right)= \infty$ for any $z_j\in \mathcal{S}$. 
Consequently, we have $\int_\mathcal{S} K_h^2\left(z_j-((\phi^*)^{-1}(\bu^+))_j\right)\mbox{d}z_j= \infty,$ which leads to a contradiction.

Using Inequalities~(\ref{s12}), Condition 6, and Lemma~\ref{lem_4}, we have
\begin{align}
 |g_z(\bu)-g_{z,{\rm est.}}(\bu)|&=\left|\prod_{j=1}^d K_h\left(z_j-((\phi^*)^{-1}(\bu))_j\right)-\prod_{j=1}^d K_h\left(z_j-(\widehat\phi^{-1}(\bu))_j\right) \right |\nonumber\\
    &\le C_1^{d-1}\sum_{j=1}^d L\|((\phi^*)^{-1}(\bu))_j-(\widehat\phi^{-1}(\bu))_j\|_2\nonumber\\
    &=C_1^{d-1}L\|((\phi^*)^{-1}(\bu))-(\widehat\phi^{-1}(\bu))\|_1,   \label{eq:s4}
\end{align}
where $\|\cdot\|_2$ and $\|\cdot\|_1$ are the $l_2$ norm and $l_1$ norm, respectively.

Combining Equations~(\ref{e33}),(\ref{e3}) and (\ref{eq:s4}), for $d\ge 3$, we have
\begin{align}
    |\pzxr-\widehat{p}_{\rm est.}(\bz)|&\le\frac{1}{r}\sum_{i=1}^r|g_z(\bu_i^*)-g_{z,{\rm est.}}(\bu_i^*)|\nonumber\\
    &\le C_1^{d-1}L\sup_{\bu\in\{\bu_i\}_{i=1}^n}\|((\phi^*)^{-1}(\bu))-((\widehat\phi)^{-1}(\bu))\|_1\nonumber\\
    &\le C_1^{d-1}L{d}\sup_{\bu\in\{\bu_i\}_{i=1}^n}\|((\phi^*)^{-1}(\bu))-((\widehat\phi)^{-1}(\bu))\|_\infty\nonumber\\
    &=O_p\left(\frac{\log(n)^{1/d}}{n^{1/d}}\right) \label{eq:s8}\\
    &=O_p\left(\frac{\log(n)^{1/d}}{\log(r)^{1/d}}\right)O_p\left(\frac{\log(r)^{1/d}}{r}\right)O_p\left(\frac{r}{n^{1/d}}\right)\nonumber\\
    &=O_p\left(\frac{\log(r)^{1/d}}{r}\right),\label{eq:s9}
\end{align}
where $\|\cdot\|_\infty$ is the $l_\infty$.
Here, Equation~(\ref{eq:s8}) comes from (\ref{eq:lems4}) in Lemma \ref{lem_5}, and Equation~(\ref{eq:s9}) comes from the assumption that $r=O(n^{1/d})$.
For the case when $d=2$, according to Lemma~\ref{lem_5}, we have
\begin{align}
    |\pzxr-\widehat{p}_{\rm est.}(\bz)|
    &=O_p\left(\frac{\log(n)^{1/d+1/4}}{n^{1/d}}\right) =O_p\left(\frac{\log(r)^{1/d+1/4}}{r}\right).\label{eq:s15}
\end{align}
Combining Equations~(\ref{eq:s9}) and (\ref{eq:s15}), for $d\ge 2$, we have
\begin{equation}\label{eq:s16}
    \mbox{E}|\pzxr-\widehat{p}_{\rm est.}(\bz)|
     =O\left(\frac{\log(r)^{1/d+1/4}}{r}\right).
\end{equation}

Next, we consider the upper bound for $(\pzxr-\mbox{E}\left(\pzx\right))^2$.
Combining the results in Equations~(\ref{e2}), (\ref{e3}) and Lemma~\ref{lem_1}, we have,
\begin{align}
\left|\pzxr-\mbox{E}\left(\pzx\right) \right|
& = \left|\frac{1}{r}\sum_{i=1}^r g_{z}(\bu_i^*)-\int_{[0,1]^d}g_{z}(\bu)\mbox{d}\bu\right| \leq D^*(\U_r^*)\mathcal{V}(g_{\bz})\label{e0}.
\end{align}

Following the definition of the total variation, we have
\[
\mathcal{V}(g_{z}) = \int_{[0,1]^d}\|\bigtriangledown g_{z}(\bu)\|\mbox{d}\bu,
\]
where $\|\cdot\|$ is the $l_2$ norm, and $\bigtriangledown g_{z}(\bu) = \left(\frac{\partial g_{z}(\bu)}{\partial u_1},..., \frac{\partial g_{z}(\bu)}{\partial u_d}\right)^T$. 
To simplify the expression of $g_z(\bu)$, we let
\[
\mathcal{K}(\x) = \prod_{j=1}^d K(x_j), \quad \x \in \mathbb{R}^d.
\]
One thus has $g_z(\bu)=\frac{1}{h^d}\mathcal{K}\left(\frac{\bz-(\phi^*)^{-1}(\bu)}{h}\right)$. Let $\bm\omega = \frac{\bz-(\phi^*)^{-1}(\bu)}{h}$, we have
\[
\bigtriangledown g_{z}(\bu) = \frac{1}{h^d}J_{\bm\omega \to \bu}^T\bigtriangledown \mathcal{K}(\bm\omega),
\]
where 
\[
J_{\bm\omega \to \bu}=
\begin{bmatrix}
    \frac{\partial \omega_1}{\partial u_1} & \dots  & \frac{\partial \omega_1}{\partial u_d} \\
    \vdots & \ddots & \vdots \\
    \frac{\partial \omega_d}{\partial u_1} & \dots  & \frac{\partial \omega_d}{\partial u_d}
\end{bmatrix}.
\]
Similarly, we define
\[
J_{\bu \to \bm\omega}=
\begin{bmatrix}
    \frac{\partial u_1}{\partial \omega_1} & \dots  & \frac{\partial u_1}{\partial \omega_d} \\
    \vdots & \ddots & \vdots \\
    \frac{\partial u_d}{\partial \omega_1} & \dots  & \frac{\partial u_d}{\partial \omega_d}
\end{bmatrix},
J_{\phi^*}=
\begin{bmatrix}
    \frac{\partial (\phi^*(\x))_1}{\partial x_1} & \dots  & \frac{\partial (\phi^*(\x))_1}{\partial x_d} \\
    \vdots & \ddots & \vdots \\
    \frac{\partial (\phi^*(\x))_d}{\partial x_1} & \dots  & \frac{\partial (\phi^*(\x))_d}{\partial x_d}
\end{bmatrix}, 
\]
and
\[
J_{(\phi^*)^{-1}}=
\begin{bmatrix}
    \frac{\partial ((\phi^*)^{-1}(\bu))_1}{\partial u_1} & \dots  & \frac{\partial ((\phi^*)^{-1}(\bu))_1}{\partial u_d} \\
    \vdots & \ddots & \vdots \\
    \frac{\partial ((\phi^*)^{-1}(\bu))_d}{\partial u_1} & \dots  & \frac{\partial ((\phi^*)^{-1}(\bu))_d}{\partial u_d}
\end{bmatrix}.
\]

Notice that $J_{\bm\omega \to \bu} = -\frac{1}{h}J_{(\phi^*)^{-1}}$, one thus has
\[
\bigtriangledown g_{z}(\bu) = \frac{1}{h^{d+1}}J_{(\phi^*)^{-1}}^T\bigtriangledown \mathcal{K}(\bm\omega).
\]
Using the Jensen's inequality, we have
\begin{align}
\mathcal{V}^2(g_{z}) & \leq\int_{[0,1]^d}||\bigtriangledown g_{z}(\bu)||^2\mbox{d}\bu\nonumber\\
& = \int_{[0,1]^d}(\bigtriangledown g_{z}(\bu))^T\bigtriangledown g_{z}(\bu)\mbox{d}\bu\nonumber\\
& = \frac{1}{h^{2d+2}}\int_{[0,1]^d}(\bigtriangledown \mathcal{K}(\bm\omega))^TJ_{(\phi^*)^{-1}}J_{(\phi^*)^{-1}}^T\bigtriangledown \mathcal{K}(\bm\omega)\mbox{d}\bu\nonumber\\
& = \frac{1}{h^{2d+2}}\int_{\Omega}(\bigtriangledown \mathcal{K}(\bm\omega))^TJ_{(\phi^*)^{-1}}J_{(\phi^*)^{-1}}^T\bigtriangledown \mathcal{K}(\bm\omega)|\det(J_{\bu \to \bm\omega})|d\bm\omega\nonumber\\
& = \frac{1}{h^{d+2}}\int_{\Omega}(\bigtriangledown \mathcal{K}(\bm\omega))^TJ_{(\phi^*)^{-1}}J_{(\phi^*)^{-1}}^T\bigtriangledown \mathcal{K}(\bm\omega)|\det(J_{\phi^*})|d\omega\label{e4},
\end{align}
where the fact that $\bu = \phi^*(\bz-h\bm\omega)$, $J_{\bu \to \bm\omega} = -hJ_{\phi^*}$, and $|\det(J_{\bu \to \bm\omega})| = h^d|\det(J_{\phi^*})|$ are used in the last equation.

Notice that 
\begin{align}
(\bigtriangledown \mathcal{K}(\bm\omega))^TJ_{(\phi^*)^{-1}}J_{(\phi^*)^{-1}}^T\bigtriangledown \mathcal{K}(\bm\omega) & = \mathrm{tr}\left((\bigtriangledown \mathcal{K}(\bm\omega))^TJ_{(\phi^*)^{-1}}J_{(\phi^*)^{-1}}^T\bigtriangledown \mathcal{K}(\bm\omega)\right)\nonumber\\
& = \mathrm{tr}\left(\bigtriangledown \mathcal{K}(\bm\omega)(\bigtriangledown \mathcal{K}(\bm\omega))^TJ_{(\phi^*)^{-1}}J_{(\phi^*)^{-1}}^T\right)\nonumber\\
& \leq \mathrm{tr}\left(\bigtriangledown \mathcal{K}(\bm\omega)(\bigtriangledown \mathcal{K}(\bm\omega))^T\right)\mathrm{tr}\left(J_{(\phi^*)^{-1}}J_{(\phi^*)^{-1}}^T\right)\label{e5}.
\end{align}

For the first term in the right-hand-side of Inequality~(\ref{e5}), i.e., $\mathrm{tr}\left(\bigtriangledown \mathcal{K}(\bm\omega)(\bigtriangledown \mathcal{K}(\bm\omega))^T\right)$, we have
\begin{align}
\mathrm{tr}\left(\bigtriangledown \mathcal{K}(\bm\omega)(\bigtriangledown \mathcal{K}(\bm\omega))^T\right) & = \mathrm{tr}\left((\bigtriangledown \mathcal{K}(\bm\omega))^T\bigtriangledown \mathcal{K}(\bm\omega)\right)\nonumber\\
& = (\bigtriangledown \mathcal{K}(\bm\omega))^T\bigtriangledown \mathcal{K}(\bm\omega)\nonumber\\
& = \sum_{k=1}^d\left(\left\{\prod_{j\neq k}K^2(\omega_j)\right\}\left(K'(\omega_k)\right)^2\right)\label{e66}.
\end{align}

For the second term in the right-hand-side of Inequality~(\ref{e5}), we have
\begin{eqnarray}\label{e6}
\mathrm{tr}\left(J_{(\phi^*)^{-1}}J_{(\phi^*)^{-1}}^T\right)\leq C,
\end{eqnarray}
for a positive constant $C$.
This is because $(\phi^*)^{-1}$ is an optimal transport map that defined on a bounded domain $[0,1]^d$.
Furthermore, according to Lemma S3, the derivative of $(\phi^*)^{-1}$ is continuous.
Consequently, all the entries in $J_{(\phi^*)^{-1}}$ are finite, and thus Inequality~(\ref{e6}) can be satisfied.
Plugging Equation~(\ref{e66}) and Inequality~(\ref{e5}) back into Equation~(\ref{e4}), we have
\begin{align}
\mathcal{V}^2(g_{z}) & \leq \frac{1}{h^{d+2}}C\int\cdots\int\sum_{k=1}^d\left(\left\{\prod_{j\neq k}K^2(\omega_j)\right\}\left(K'(\omega_k)\right)^2\right)d\omega_1\cdots d\omega_d\nonumber\\
& = \frac{1}{h^{d+2}}C\sum_{k=1}^d\left\{\prod_{j\neq k}\int_{\Omega_j}K^2(\omega_j)d\omega_j\int_{\Omega_k}\left(K'(\omega_k)\right)^2d\omega_k\right\}\nonumber\\
& = O\left(\frac{1}{h^{d+2}}\right)\label{e7}.
\end{align}

Combining Inequalities~(\ref{e7}) and ~(\ref{e0}), we have 
\begin{eqnarray}\label{e8}
\mbox{E}\Big(\pzxr-\mbox{E}\left(\pzx\right) \Big)^2 \leq \Big(D^*(\U_r^*)\Big)^2\mathcal{V}^2(g_{z}) = O\left(\frac{1}{r^{2(1-\delta)} h^{d+2}}\right).
\end{eqnarray}


Plugging (\ref{eq:s16}), (\ref{e8}) and (\ref{e9}) into (\ref{eq:s92}) yields  
\begin{eqnarray*}
\mbox{MSE}(\pzxr) &=&  O\left(\frac{\log(r)^{2/d+1/2}}{r^2}\right)+ O{\left(\frac{1}{r^{2(1-\delta)} h^{d+2}}\right)} + O(h^4)\\
&=&O{\left(\frac{1}{r^{2(1-\delta)} h^{d+2}}\right)} + O(h^4).
\end{eqnarray*}
Consequently, when $h=O(r^{-\frac{2(1-\delta)}{6+d}})$, we have
\begin{equation*}
\mbox{MSE}(\pzxr) = O(r^{-\frac{8(1-\delta)}{6+d}}).
\end{equation*}


\section{CPU time comparison}

We compare the CPU time for different subsampling methods. 
We consider $n=10^4,10^5, 10^6$, $p=10$, and we set the subsample size $r\approx n^{1/2}$.
The results are shown in Table~\ref{table1}.
We observe that our method is reasonably efficient even when $n=10^6$.
These results indicate the proposed method is able to provide effective and efficient subsamples for the datasets of large $n$.

\begin{table}[h!]\caption{Average CPU time (in seconds) for subsampling}\label{table1}
\centering
\begin{tabular}{ cccc }
\hline
 & SPARTAN & SP & KM \\
\hline
$n=10^4$ & 0.33 & 2.11 & 0.44 \\
$n=10^5$ & 2.26 & 5.42 & 19.50 \\
$n=10^6$ & 21.19 & 18.95 & 1167.83 \\
\hline
\end{tabular}
\end{table}

\clearpage
\bibliographystyle{agsm}
\bibliography{ref}